\documentclass[12pt]{article}

\usepackage{graphicx,psfrag,epsf}
\usepackage{enumerate}
\usepackage{natbib}
\usepackage{mathtools}
\usepackage{booktabs}
\usepackage{color}
\usepackage[english]{babel} 
\usepackage[protrusion=true,expansion=true]{microtype} 
\usepackage{amsmath,amsfonts,amsthm}
\usepackage{comment}
\usepackage{dsfont}
\usepackage{amssymb}
\usepackage{chngcntr}
\usepackage[toc,page]{appendix}
\usepackage{textcomp}
\usepackage{url}
\usepackage{hyperref}
\usepackage{subcaption}  
\usepackage{array}
\usepackage{booktabs} 
\usepackage{setspace}
\usepackage{amsmath}
\usepackage{xcolor}
\usepackage{wrapfig}

\usepackage{soul}
\usepackage{multirow}
\usepackage{enumitem}
\usepackage{microtype} 
\usepackage[font = small,labelfont=bf,textfont=it]{subcaption}
\usepackage{xcolor}
\usepackage{tabularx}
\usepackage[font = small,labelfont=bf,textfont=it]{caption} 
\usepackage{footnote}
\usepackage{algorithm}
\usepackage[noend]{algpseudocode}
\usepackage{etoolbox}
\usepackage{tabularx}
\usepackage{authblk}
\usepackage{caption}
\usepackage{indentfirst}
\usepackage{enumitem}
\usepackage{textcomp}
\usepackage{graphicx}
\usepackage{caption}
\usepackage{subcaption}

\newlist{steps}{enumerate}{1}
\setlist[steps, 1]{label = Step \arabic*:}

\algblock{ParFor}{EndParFor}
\algnewcommand\algorithmicparfor{\textbf{for}}
\algnewcommand\algorithmicpardo{\textbf{do\ parallel}}
\algnewcommand\algorithmicendparfor{\textbf{end\ parallel\ for}}
\algrenewtext{ParFor}[1]{\algorithmicparfor\ #1\ \algorithmicpardo}
\algrenewtext{EndParFor}{\algorithmicendparfor}

\makeatletter
\def\BState{\State\hskip-\ALG@thistlm}

\newcommand{\distas}[1]{\mathbin{\overset{#1}{\kern\z@\sim}}}%

\newsavebox{\mybox}\newsavebox{\mysim}
\newcommand{\distras}[1]{%
  \savebox{\mybox}{\hbox{\kern3pt$\scriptstyle#1$\kern3pt}}%
  \savebox{\mysim}{\hbox{$\sim$}}%
  \mathbin{\overset{#1}{\kern\z@\resizebox{\wd\mybox}{\ht\mysim}{$\sim$}}}%
}
\newtheorem{theorem}{Theorem}

\newtheorem{proposition}[theorem]{Proposition}
\newtheorem{corollary}{Corollary}
\newtheorem{assumption}{Assumption}

\newcommand{\be}{\begin{equation}}
\newcommand{\ee}{\end{equation}}
    \newcommand{\bi}{\begin{itemize}}
\newcommand{\ei}{\end{itemize}}
\newcommand{\ben}{\begin{enumerate}}
\newcommand{\een}{\end{enumerate}}

\newcolumntype{K}[1]{\geq {\centering\arraybackslash}p{#1}}
\allowdisplaybreaks
\DeclareMathOperator*{\argmin}{\arg\!\min}
\makeatother

\let\oldbibliography\thebibliography
\renewcommand{\thebibliography}[1]{\oldbibliography{#1}
\setlength{\itemsep}{0pt}} 

\newcommand{\blind}{1}

\patchcmd{\footnotemark}{\stepcounter{footnote}}{\refstepcounter{footnote}}{}{}
\setcitestyle{square}
\newtheorem{lemma}{Lemma}

\begin{document}

\def\spacingset#1{\renewcommand{\baselinestretch}%
{#1}\small\normalsize} \spacingset{1}

\if1\blind
{
  \title{\bf Efficient optimization of expensive black-box simulators via marginal means, with application to neutrino detector design}
  \small
   \author{Hwanwoo Kim\footnote{Department of Statistical Science, Duke University}\;, Simon Mak$^*$\footnote{SM and HK are supported by NSF CSSI 2004571, NSF DMS 2210729, 2316012 and DE-SC0024477.}, Ann-Kathrin Schuetz\footnote{Nuclear Science Division, Lawrence Berkeley National Laboratory}\;, Alan Poon$^\ddagger$\footnote{Lawrence Berkeley National Laboratory is operated by the University of California under the U.S. Department of Energy Federal Prime Agreement DE-AC02-05CH11231.}
   }
  \maketitle
} \fi

\if0\blind
{
  \bigskip
  \bigskip
  \bigskip
  \begin{center}
    {\LARGE\bf Efficient optimization of expensive black-box simulators via marginal means, with application to neutrino detector design}
\end{center}

  \medskip
} \fi

\begin{abstract}
With advances in scientific computing, computer experiments are increasingly used for optimizing complex systems. However, for modern applications, e.g., the optimization of nuclear physics detectors, each experiment run can require hundreds of CPU hours, making the optimization of its black-box simulator $f$ over a high-dimensional space $\mathcal{X}$ a challenging task. Given limited runs at inputs $\mathbf{x}_1, \cdots, \mathbf{x}_n \in \mathcal{X}$, the best solution from these evaluated inputs can be far from optimal, particularly as dimensionality increases. Existing black-box methods, however, largely employ this ``pick-the-winner'' (PW) solution, which leads to mediocre optimization performance. To address this, we propose a new Black-box Optimization via Marginal Means (BOMM) approach. The key idea is a new estimator of a global optimizer $\mathbf{x}^*$ that leverages the so-called marginal mean functions, which can be efficiently inferred with limited runs in high dimensions. Unlike PW, this estimator can select solutions beyond evaluated inputs for improved optimization performance. Assuming $f$ follows a generalized additive model with unknown link function and under mild conditions, we prove that the BOMM estimator not only is consistent for optimization, but also has an optimization rate that tempers the ``curse-of-dimensionality'' faced by existing methods, thus enabling better performance as dimensionality increases. We present a practical framework for implementing BOMM using the transformed Gaussian process surrogate model in \cite{lin2020transformation}. Finally, we demonstrate the effectiveness of BOMM in numerical experiments and an application on neutrino detector optimization in nuclear physics.
\end{abstract}

\vfill

\newpage
\spacingset{1.55} 

\section{Introduction}\label{sec:Intro}

Scientific computing is undergoing rapid development. With recent progress, complex phenomena, e.g., rocket engines \citep{mak2018efficient}, universe expansion \citep{kaufman2011efficient} and particle collisions \citep{ji2023graphical,ji2024conglomerate}, can now be reliably simulated via virtual simulation. These ``computer experiments'' \citep{gramacy2020surrogates, deng2025design} offer an appealing alternative to physical experiments \citep{wu2011experiments}, which may be impractical or infeasible in modern applications. However, such virtual experiments often incur high computational costs that hamper their use for scientific decision-making, particularly for optimizing the simulated response surface $f(\cdot)$ over a design space $\mathcal{X}$. We face this bottleneck in our motivating application of designing complex detectors for neutrinoless double-beta decay searches \citep{Dolinski:2019nrj}. Such a decay mechanism provides important insight into the fundamental matter-antimatter asymmetry in the Universe \citep{canetti2012matter}, but its detection requires careful detector optimization to suppress cosmogenic backgrounds. While virtual simulators provide an appealing strategy for detector optimization, the simulation of a single detector design can require hundreds of CPU hours, which makes its optimization a highly challenging task.

A proven solution is probabilistic surrogate modeling \citep{overstall2016multivariate}. The idea is to run the computer experiment at designed input points $\mathbf{x}_1, \cdots, \mathbf{x}_n \in \mathcal{X} \subset \mathbb{R}^d$, then use the simulated data $[f(\mathbf{x}_i)]_{i=1}^n$ to fit a probabilistic model that predicts $f$ with uncertainty at untested inputs. A popular surrogate choice is the Gaussian process (GP; \citealp{RasmussenWilliams06, stein2012interpolation}), which provides flexible probabilistic modeling with closed-form predictive equations. This not only permits efficient exploration of $f$ over the design space $\mathcal{X}$, but also facilitates timely downstream scientific decision-making, e.g., optimization \citep{miller2025targeted, kim2024enhancing} and inverse problems \citep{ehlers2024bayesian, kim2024optimization}. Recent developments on GP surrogates include the use of deeper architectures \citep{sauer2023active,montagna2016computer} and the incorporation of domain physics \citep{ding2025bdrymat,golchi2015monotone}.

We consider the specific task of minimizing\footnote{Here, one can easily maximize $f$ by minimizing the modified objective $-f$.} the expensive black-box function $f$:
\begin{equation}
    \mathbf{x}^* \in \underset{\mathbf{x} \in \mathcal{X}}{\text{Argmin}} \; f(\mathbf{x}),
    \label{eq:min}
\end{equation}
which is critical for many facets of decision-making via computer experiments, including system optimization \citep{paulson2025bayesian} and control \citep{miller2024diverse}. Here, $\text{Argmin}$ denotes the set of input points that minimize $f$. Existing ``black-box optimization'' approaches can be classified as sequential or one-shot methods. Sequential methods perform sequential (or batch-sequential) evaluations of $f$, where each input $\mathbf{x}_n$ is adaptively selected using evaluation data from previous inputs $\mathbf{x}_1, \cdots, \mathbf{x}_{n-1}$. Such methods have received much attention in the Bayesian optimization literature; see, e.g., \cite{jones1998efficient,chen2023,frazier2008knowledge}. However, for expensive computer simulators, the high cost of a single run can be a barrier for sequential methods. For example, in our detector optimization application, a high-fidelity simulation for a single detector design can require hundreds of CPU hours, which prevents any adaptive iterations when a decision needs to be made promptly. In such a scenario, \textit{one-shot} methods that simultaneously perform all runs may be more feasible. One-shot methods are facilitated by the rise of distributed computing, which permits the simultaneous evaluation of $f$ at many inputs via multi-core processing. We will focus on such one-shot methods here, as motivated by our application.

Existing one-shot black-box optimization approaches broadly fall into two categories \citep{thomaser2022one}. The first adopts the simple but intuitive strategy of picking the best solution $\hat{\mathbf{x}}_n^* = \underset{\mathbf{x} \in \{\mathbf{x}_1, \cdots, \mathbf{x}_n\}}{\arg\min} f(\mathbf{x})$ amongst the evaluated inputs. This was coined the ``pick-the-winner'' (PW) approach in \cite{wu1990sel} and \cite{mak2019analysis}, and is broadly used in practice. Given limited runs over a high-dimensional space $\mathcal{X}$, however, the evaluated inputs can be far from optimal, in which case PW may yield mediocre performance. The second strategy is to first fit a surrogate model $\hat{f}_n(\cdot)$ from data, then ``infer'' $\mathbf{x}^*$, i.e., infer an optimal solution from \eqref{eq:min}, via its minimizer $\hat{\mathbf{x}}_n^* = \underset{\mathbf{x}\in \mathcal{X}}{\arg\min} \; \hat{f}_n(\mathbf{x})$. While this surrogate-based approach may yield improvements over PW when the surrogate fits well globally, this is by no means guaranteed; when this fit is poor, such approaches may perform worse than PW. For high-dimensional spaces $\mathcal{X}$, surrogate-based approaches may further face a ``curse-of-dimensionality'' \citep{bellman1966dynamic}, in that the surrogate fit becomes increasingly poor as dimension $d$ increases. This is well-known for GP surrogates, which have an $L_\infty$-prediction rate of $\mathcal{O}(n^{-\nu/d})$ using the Mat\'ern kernel \citep{stein2012interpolation} with smoothness parameter $\nu > 0$; see \cite{wu1993local,wendland2004scattered}. This exponential dependence of sample size $n$ on $d$ can result in rapid deterioration of surrogate (and thus optimization) performance as dimensionality increases \citep{ding2019bdrygp}. A similar curse-of-dimensionality is also present for sequential Bayesian optimization methods \citep{bull2011convergence, kim2025inexact}.

To address this, we propose a new Black-box Optimization via Marginal Means (BOMM) approach for one-shot black-box optimization. The key idea is to construct a new BOMM estimator $\hat{\mathbf{x}}_n^*$ of an optimizer $\mathbf{x}^*$ that depends on the so-called marginal mean functions. In contrast to PW, our BOMM estimator can select solutions beyond evaluated inputs to improve black-box optimization with limited data. In contrast to surrogate-based approaches, which require the challenging task of a good surrogate fit over the \textit{full} domain $\mathcal{X}$, the marginal mean functions in BOMM can be effectively estimated in high dimensions with limited runs. Assuming $f$ follows a generalized additive model \citep{hastie2017generalized} with unknown link function and under mild regularity conditions, we prove that the BOMM optimality gap $|f(\hat{\mathbf{x}}_n^*) - f(\mathbf{x}^*)|$ not only converges to zero, but does so at a rate with considerably less dependence on dimensionality than existing methods, thus tempering the curse-of-dimensionality and facilitating good performance as $d$ increases. We then present a methodological framework, which leverages the transformed approximate additive GP model in \cite{lin2020transformation} for an effective implementation of BOMM. Finally, we demonstrate the effectiveness of BOMM over the state-of-the-art in a suite of numerical experiments and for our motivating application of neutrino detector optimization. 

There are important practical considerations when inferring an optimal solution beyond evaluated points. Despite its limitations, one appeal of PW is its reliability: its inferred solution $\hat{\mathbf{x}}_n^*$ is naturally validated by an evaluated point. This is desirable in applications where final design decisions are made promptly after inference. For complex scientific applications (e.g., detector design), however, the inferred solution $\hat{\mathbf{x}}_n^*$ is typically but one step in the design process; such a solution is then further investigated and validated by scientists prior to design decisions. For such problems, the validation of $\hat{\mathbf{x}}_n^*$ within the black-box optimization procedure is not essential, and the improvement gained from inferring beyond evaluated points can be highly beneficial with limited runs, as we show later.

This paper is organized as follows. Section \ref{sec:Background} provides background on GPs, existing one-shot black-box methods, and their potential limitations in motivating experiments. Section \ref{sec:bomm} presents the proposed BOMM estimator and proves its optimization consistency and associated rate. Section \ref{sec:bommimpl} outlines a comprehensive methodological framework for effective implementation. Sections \ref{sec:numerics} and \ref{sec:application} investigate the performance of BOMM in numerical experiments and an application on detector optimization. Section \ref{sec:conclusion} concludes the paper.

\section{Background and Motivation}\label{sec:Background}
We first give a brief review of GPs, then outline existing one-shot black-box optimization methods and their potential limitations in a motivating experiment.

\subsection{Gaussian process modeling}
\label{sec:gp}
Let $f:\mathcal{X} \rightarrow \mathbb{R}$ be the black-box function to optimize, where $\mathcal{X}$ is its design space. In what follows, we presume $\mathcal{X}$ to be a rectangular domain of the form $\mathcal{X} = \prod_{l=1}^d [L_l,U_l]$, where $L_l$ and $U_l$ are the lower and upper limits for the $l$-th input variable. Given the black-box nature of $f$, one can adopt a Gaussian process (GP; \citealp{RasmussenWilliams06}) prior on $f$: $f(\cdot) \sim \text{GP}\{\mu,k(\cdot,\cdot)\}$. Here, $\mu$ is a mean parameter that can be estimated from data, and $k(\cdot,\cdot)$ is a kernel function that controls sample path smoothness. Common kernel choices include the squared-exponential and the Mat\'ern kernels \citep{stein2012interpolation,gramacy2020surrogates}. 

Next, suppose the expensive computer simulator is evaluated at $n$ designed input points $\mathbf{x}_1, \cdots, \mathbf{x}_n$, yielding data $\mathbf{f}_n = [f(\mathbf{x}_1), \cdots, f(\mathbf{x}_n)]$. In what follows, we presume that the simulator is deterministic, in that it returns the same output $f(\mathbf{x})$ given the same input $\mathbf{x}$. This is commonly assumed in the computer experiments literature, particularly when $f$ solves a deterministic partial differential equation system. One can easily account for Gaussian simulation noise by incorporating a nugget term in the predictive equations below; see \cite{peng2014choice}. Conditional on data $\mathbf{f}_n$, the predictive distribution of $f(\mathbf{x}_{\rm new})$ at an untested point $\mathbf{x}_{\rm new}$ can be shown to be $[f(\mathbf{x}_{\rm new})|\mathbf{f}_n] \sim \mathcal{N}\{\hat{f}_n(\mathbf{x}_{\rm new}),\sigma^2_n(\mathbf{x}_{\rm new})\}$, where:
\begin{align}
\begin{split}
    \hat{f}_n(\mathbf{x}_{\rm new})&=\mu+\mathbf{k}_n^T(\mathbf{x}_{\rm new})\mathbf{K}_n^{-1}(\mathbf{f}_n-\mu\boldsymbol{1}),\\
    \sigma^2_n(\mathbf{x}_{\rm new})&=k(\mathbf{x}_{\rm new},\mathbf{x}_{\rm new})-\mathbf{k}_n^T(\mathbf{x}_{\rm new})  \mathbf{K}_n^{-1}\mathbf{k}_n(\mathbf{x}_{\rm new}).
    \label{eq:gppred}
\end{split}
\end{align}
Here, $\mathbf{K}_n = [k(\mathbf{x}_i,\mathbf{x}_j)]_{i,j=1}^n$ and $\mathbf{k}_n(\mathbf{x}_{\rm new}) = [k(\mathbf{x}_{\rm new},\mathbf{x}_i)]_{i=1}^n$. Equation \eqref{eq:gppred} provides the basis for efficient probabilistic surrogate modeling of $f(\cdot)$ over the input space $\mathcal{X}$.

\subsection{Existing one-shot black-box optimization methods}
\label{sec:exist}

As mentioned in the Introduction, existing one-shot black-box optimization methods can be broadly categorized as pick-the-winner and surrogate-based approaches; these approaches differ in how they ``estimate''\footnote{For black-box optimization, the quality of an estimator $\hat{\mathbf{x}}_n^*$ for an optimal solution $\mathbf{x}^*$ is typically gauged by its optimality gap $f(\hat{\mathbf{x}}_n^*) - f(\mathbf{x}^*)$.} an optimal solution $\mathbf{x}^*$. PW-based approaches \citep{wu1990sel} are simple but intuitive: they select the best observed solution $\hat{\mathbf{x}}_n^* = \underset{\mathbf{x} \in \{\mathbf{x}_1, \cdots, \mathbf{x}_n\}}{\arg\min} f(\mathbf{x})$ amongst the evaluated points. The PW estimator of $\mathbf{x}^*$ is commonly used in practice. One reason is that such an estimator is ``robust'' \citep{mak2019analysis}, in that it does not select points on which $f$ has not been evaluated. However, given a limited sample size $n$ (due to the costly nature of $f$), the evaluated design points can be far from optimal, meaning the PW estimator may yield mediocre optimization performance.

Surrogate-based optimization (SBO) approaches employ an alternate estimator of $\mathbf{x}^*$. One first uses the collected data on $f$ to train a surrogate model $\hat{f}_n(\cdot)$, then selects the optimizer of this surrogate $\hat{\mathbf{x}}_n^* = \underset{\mathbf{x}\in \mathcal{X}}{\arg\min} \; \hat{f}_n(\mathbf{x})$ as its estimate of $\mathbf{x}^*$. When the trained surrogate $\hat{f}_n$ fits well globally, surrogate-based approaches can provide improved optimization over PW \citep{thomaser2022one}; when this is not the case, however, such approaches may perform worse than PW. This phenomenon is exacerbated when $\mathcal{X}$ is high-dimensional, where surrogate quality can deteriorate quickly given a limited sample size $n$ \citep{ding2019bdrygp}. This ``curse-of-dimensionality'' is well-known for GP surrogates: for a GP with an isotropic Mat\'ern kernel $k$ \citep{stein2012interpolation} and smoothness parameter $\nu>0$ (we call this the ``Mat\'ern-$\nu$ GP'' later), one can show \citep{wu1993local,wendland2004scattered} that its $L_\infty$-prediction rate is $\|f-\hat{f}_n\|_{\infty} = \mathcal{O}(n^{-\nu/d})$ with optimally selected design points, where $f$ is in the reproducing kernel Hilbert space (RKHS; \citealp{aronszajn1950theory}) for kernel $k$, denoted $\mathcal{F}$. The optimality gap using such a surrogate thus follows a similar rate of $|f(\hat{\mathbf{x}}_n^*)-f(\mathbf{x}^*)| = \mathcal{O}(n^{-\nu/d})$ for $f \in \mathcal{F}$. The exponential dependence of sample size $n$ on dimension $d$ in this rate suggests that the performance of SBO methods can quickly worsen as dimension increases.

\subsection{Motivating experiments}\label{subsec:motiv_prob}

To highlight these limitations of PW-based and surrogate-based approaches for one-shot black-box optimization, we explore two motivating experiments in the challenging setting with limited runs in a (moderately) high-dimensional space. We consider two test functions in the computer experiments literature \citep{surjanovic2013virtual}: the six-hump camel function in $d=6$ dimensions, and the wing weight function in $d=10$ dimensions; their specific forms are provided in Appendix G. For each function, we take a one-shot design of $n = 10d$ points from a maximin Latin hypercube design \citep{morris1995exploratory} scaled to $\mathcal{X}$. For SBO, we consider two surrogate choices: a GP surrogate using the square-exponential kernel (SBO-SqExp), and a deep GP \citep{sauer2023active} surrogate (SBO-DGP). Experimental details are provided later in Section \ref{sec:numerics}.

\begin{wrapfigure}{r}{0.60\textwidth}
    \centering
    \vspace{-0.3cm}
    \includegraphics[width=\linewidth]{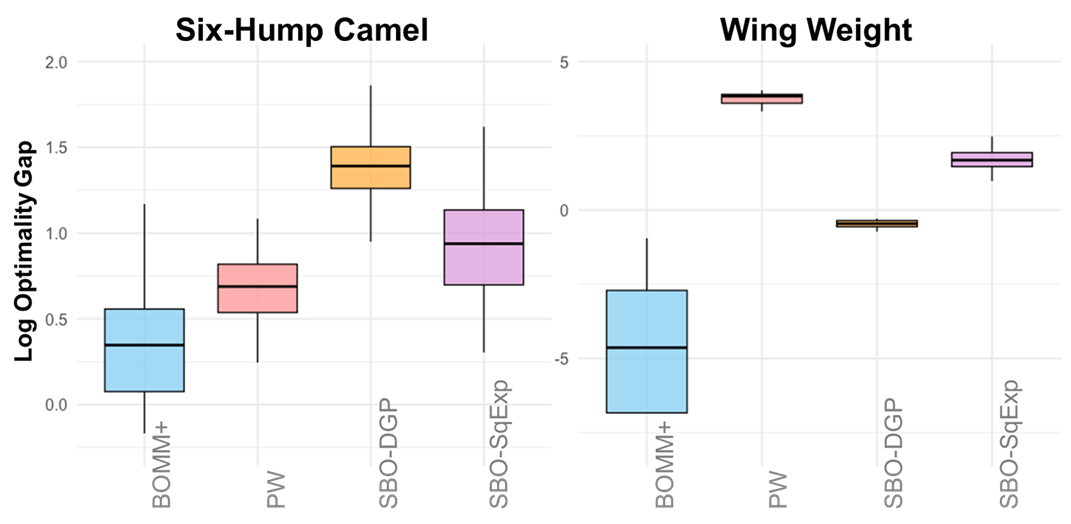}
    \caption{Log-optimality gaps of the compared methods for the six-hump camel and wing weight functions. Boxplots show experiment variability over 20 replications for each method.}
    \vspace{-0.3cm}
    \label{Fig:mot}
\end{wrapfigure}

Figure \ref{Fig:mot} shows the boxplots of the log-optimality gaps $\log |f(\hat{\mathbf{x}}_n^*)-f(\mathbf{x}^*)|$ for each method over 20 replications. There are several observations to note. First, the simple PW estimator yields large optimality gaps for the 10-d wing weight experiment. This is not surprising, as the limited evaluated points are likely far from optimal, particularly on a $d=10$-dimensional space $\mathcal{X}$. Second, the surrogate-based optimizers perform better than PW for the wing weight function, but worse for the six-hump camel function. A plausible reason is that the latter function is more complex over its domain: given a small sample size, a good global surrogate fit becomes more challenging, resulting in worse surrogate-based-optimization performance. This reliance on a good global surrogate fit can make SBO methods unreliable, particularly with limited runs in moderate-to-high dimensions. To foreshadow, the proposed BOMM addresses these limitations (see Figure \ref{Fig:mot}) via a new estimator for $\mathbf{x}^*$ that relies on marginal mean functions, which can be effectively estimated from limited data in high dimensions; we explore this next.

\section{Black-box Optimization via Marginal Means}\label{sec:bomm}

We present next our BOMM framework, which employs a new estimator for $\mathbf{x}^*$ using marginal mean functions. We first outline its estimation framework, then prove its optimization consistency and associated rate under mild regularity conditions. Such a rate tempers the curse-of-dimensionality noted earlier for existing black-box methods, enabling better optimization performance as $d$ increases. A methodological framework for robust implementation is presented later in Section \ref{sec:bommimpl}.

\subsection{The BOMM estimator}
\label{sec:estim}

Suppose the black-box function $f$ follows the general model:
\begin{equation}
f(\mathbf{x)} = \phi \circ h(\mathbf{x}), \quad h(\mathbf{x}) = h_1(x_1) + \cdots + h_d(x_d) + \zeta(\mathbf{x}),
\label{eq:gam}
\end{equation}
where $\phi \circ h (\mathbf{x}) = \phi\{h(\mathbf{x})\}$ denotes the composition of functions $\phi$ and $h$. Here, $\phi$ is a strictly monotone (and thus invertible) link function to be estimated from data, $h_1(x_1), \cdots, h_d(x_d)$ are additive functions on each input, and $\zeta(\mathbf{x})$ accounts for ``mild'' deviations from additivity for $h(\mathbf{x})$; more on this later. Without loss of generality, we presume in the following that $\phi$ is strictly monotonically increasing, as one can account for the monotonically decreasing case by reversing the sign on $h(\mathbf{x})$.


Note that, with $\zeta(\mathbf{x}) = 0$, the model \eqref{eq:gam} reduces to a \textit{generalized additive model} (GAM; \citealp{hastie2017generalized}) with unknown link function. GAMs are widely used in the statistical learning literature \citep{hastie2009elements,rudin2022interpretable} due to its flexible modeling framework and interpretability. A key appeal of a GAM is that it provides some relief from the curse-of-dimensionality for high-dimensional regression \citep{stone1986dimensionality}, by leveraging an additive structure after link transformation. Its form can further be justified via the well-known Kolmogorov-Arnold representation theorem (see, e.g., \citealp{tikhomirov1991representation}). The inclusion of $\zeta(\mathbf{x})$ enhances model flexibility by accounting for potential deviations from additivity in $h(\mathbf{x})$. This use of a carefully specified transformation for near-additive modeling has a long history in statistics, going back to the Box-Cox transformation \citep{box1964analysis} and ANOVA modeling \citep{wu2011experiments}. We adopt later a probabilistic modeling framework for \eqref{eq:gam} using the transformed approximate additive GP in \cite{lin2020transformation} to guide BOMM optimization.

Next, define the so-called transformed \textit{marginal mean} functions of $f$:
\begin{equation}
    m_l(x_l) = \int_{\mathcal{X}_{-l}} \phi^{-1} \circ f(\mathbf{x}) \; d\mathbf{x}_{-l}, \quad \quad l = 1, \cdots, d.
    \label{eq:margfunc}
\end{equation}
Here, $\mathbf{x}_{-l}$ refers to all variables in $\mathbf{x}$ except $x_l$, and $\mathcal{X}_{-l}$ denotes its domain. For an input $l$, such a function marginalizes the transformed response surface $\phi^{-1} \circ f$ over the remaining $d-1$ inputs. Given data $\mathbf{f}_n$ on $f$ at design points, let $\hat{m}_l(x_l)$ denote the estimator of this marginal mean function for input $l$; we will discuss how to construct such an estimator later. The BOMM estimator $\hat{\mathbf{x}}_n^* = (\hat{x}_{n,1}^*, \cdots, \hat{x}_{n,d}^*)$ for $\mathbf{x}^*$ then takes the following form:
\begin{equation}
\hat{x}_{n,l}^* = \underset{x_l}{\text{argmin}} \; \hat{m}_l(x_l), \quad l = 1, \cdots, d.
\label{eq:bomm}
\end{equation} 
In words, the $l$-th element of the BOMM estimator is taken as the minimizer of the estimated marginal mean function $\hat{m}_l(x_l)$ for the $l$-th input.

One way to intuit this estimator is as follows. Suppose $f$ follows a GAM (i.e., the model in \eqref{eq:gam} with $\zeta(\mathbf{x}) = 0$), and suppose its link function $\phi$ and additive functions $h_1, \cdots, h_d$ are known. Then the solution $\tilde{\mathbf{x}}^* = (\tilde{x}_1^*, \cdots, \tilde{x}_d^*)$ defined as $\tilde{x}_{l}^* = \underset{x_l}{\text{argmin}} \; m_l(x_l)$ must be a global optimizer of $f$; this follows from the fact that $\phi$ is monotonically increasing and $h(\mathbf{x})$ is additive. Given this constructive form for $\mathbf{x}^*$, the BOMM estimator \eqref{eq:bomm} targets the estimation of such a solution via the \textit{estimated} marginal mean functions $\hat{m}_l$. When such embedded near-additive structure is present, these marginal functions can be estimated efficiently even in high dimensions \citep{horowitz2007rate}; BOMM exploits this for efficient black-box optimization with limited data.

The BOMM estimator is motivated by a related problem of parameter design optimization for quality improvement \citep{wu2011experiments}. The latter targets the optimization of a physical system, e.g., the mean yield of a plot of land, under different control inputs with varying discrete levels. The goal is to identify a near-optimal input setting with limited physical experiment runs. \cite{wu1990sel} coined the term ``pick-the-winner'' as the simple strategy that selects the best observed setting within the limited runs. \cite{taguchi1986introduction} instead advocates for an alternate ``analysis of marginal means'' (AM) strategy. For each input $l$, the AM estimator selects the level that minimizes its marginal mean over such an input. The intuition is that such marginal effects in one dimension can be estimated more efficiently than the minimum over the full $d$-dimensional domain. Not surprisingly, when $f$ is near-additive (i.e., it has few interactions), AM is markedly more efficient than PW for system optimization with limited runs \citep{mak2019analysis}. Our BOMM estimator extends this for \textit{continuous} black-box optimization, coupled with a flexible \textit{generalized} additive modeling framework \eqref{eq:gam} that relaxes the near-additivity requirement on $f$.

It is also useful to contrast our approach with the earlier surrogate-based one-shot approaches, which directly optimize a standard surrogate model trained on data $\mathbf{f}_n$. As noted earlier, such approaches may yield mediocre performance given small sample sizes in high dimensions, when the surrogate fits poorly over the full space $\mathcal{X}$. Instead of relying on the full fitted surrogate, BOMM instead leverages the estimated marginal mean functions, which can be more easily inferred in high dimensions with limited data. As we see later, this can improve theoretical and empirical optimization performance by tempering the curse-of-dimensionality. A key reason lies in (i) the reduced function space for GAMs (and its generalization in \eqref{eq:gam}; see Theorem \ref{thm:gpbommconv}) compared to (ii) the highly nonparametric function spaces typically considered for surrogate modeling. Functions in the reduced space (i) permit efficient inference on marginal mean functions and its use for effective black-box optimization, whereas functions in (ii) do not permit the exploitation of such structure. Given the modeling flexibility of GAMs \citep{hastie2009elements,lin2020transformation}, this reduced space does not appear to be overly restrictive in our target problems and enables improved black-box optimization with limited data, as we see in later numerical experiments.

\subsection{Optimization consistency and rate}
\label{sec:gapanalysis}

We first investigate the convergence properties of BOMM. We will show that its optimality gap $|f(\hat{\mathbf{x}}_n^*) - f(\mathbf{x}^*)|$ converges at a rate of $\mathcal{O}_{P}(n^{-k/(4k+2)})$ when $f$ follows a GAM. Here, $k$ is the degree of differentiability on the link function $\phi$ and the additive functions $h_1, \cdots, h_d$. This considerably reduces the impact of dimensionality compared to the earlier $\mathcal{O}(n^{-\nu/d})$ rate for surrogate-based approaches that use a Mat\'ern-$\nu$ GP, thus facilitating effective optimization in high dimensions. As before, suppose the domain is $\mathcal{X} = \prod_{l=1}^d [L_l,U_l]$.

We make the following set of assumptions for theoretical analysis:

\begin{assumption}\label{assump:derivs_phi_h}
    The objective $f$ is in the form of a GAM (i.e., model \eqref{eq:gam} with $\zeta(\mathbf{x}) = 0$), with its link function $\phi$ and additive functions $h_1, \cdots, h_d$ $k$-times continuously differentiable with $k \ge 2$. Further assume:
    \begin{equation}
        \int [\phi^{(k)} (z)]^2 \; dz < \infty, \quad \int [h_l^{(k)}(x_l)]^2 \; dx_l < \infty, \quad \text{ for } l = 1, \cdots, d,
    \end{equation}
where $\phi^{(k)}$ is the $k$-th derivative of $\phi$, and the same for $h_l^{(k)}$.
\end{assumption}

\begin{assumption}\label{assump:transformation}
    The link function $\phi$ is strictly monotone increasing. 
\end{assumption}

\begin{assumption}\label{assump:design_points}
   Design points $\{\mathbf{x}_1, \cdots, \mathbf{x}_n\}$ are sampled i.i.d. from $\textup{Uniform}(\mathcal{X})$. 
\end{assumption}

\noindent Assumption \ref{assump:derivs_phi_h} provides necessary smoothness conditions on $\phi$ and $h_1, \cdots, h_d$, following \cite{horowitz2007rate}. Assumption \ref{assump:transformation} follows from the discussion in Section \ref{sec:estim}. Assumption \ref{assump:design_points} is a typical design assumption for theoretical analysis.

In the following analysis, we adopt the inference approach in \cite{horowitz2007rate} for estimating $\phi$ and $h_1, \cdots, h_d$ in a GAM (i.e., model \eqref{eq:gam} with $\zeta(\mathbf{x}) = 0$). There, these functions are jointly estimated via the constrained regularized least squares problem:
\begin{align}
\begin{split}
    \left(\widehat \phi, \widehat h_1, \cdots, \widehat h_d\right) = \underset{\phi, h_1, \cdots, h_d}{\arg\min} \; &\frac{1}{n} \sum_{i=1}^n \left\{f(\mathbf{x}_i) - \phi\left[h_1(x_{i,1}) + \cdots + h_d(x_{i,d})\right]\right\}^2  \\
    & \quad \quad + \lambda_n^2 \left( \left\{\int [\phi^{(k)}(z)]^2dz\right\}^{\nu_1/2} + \left\{\int [\phi'(z)]^2 dz \right\}^{\nu_2/2} \right),
    \label{NONPARAM_EST_HOROWITZ}
\end{split}
\end{align}
under the constraints:
\begin{align}
\sum_{l=1}^d \left[\int [h_l^{(k)}(x_l)]^2dx_l + \int [h_{l}'(x_l)]^2 dx_l \right] = 1, \quad \phi'(z) > 0,
\label{NONPARAM_EST_HOROWITZ2}
\end{align}
where $\nu_1 > 0$ and $\nu_2 > 0$ are fixed constants with $\nu_2 \ge \nu_1$. Here, the second term in \eqref{NONPARAM_EST_HOROWITZ} provides regularization on the smoothness of $\phi$ with penalty $\lambda_n$, and the first constraint in \eqref{NONPARAM_EST_HOROWITZ2} provides similar regularity on the additive functions $h_1, \cdots, h_d$. The second constraint in \eqref{NONPARAM_EST_HOROWITZ2} ensures the estimated $\phi$ is strictly monotone increasing. Following \cite{horowitz2007rate}, we adopt the following assumption on the penalty $\lambda_n$:
\begin{assumption}\label{assump:reg_parameter}
    $\lambda_n = \mathcal{O}_{P}\left(n^{-k/(2k+1)}\right)$ and $\lambda_n^{-1} = \mathcal{O}_{P}\left(n^{k/(2k+1)}\right)$.
\end{assumption}

With this, we now investigate the optimization performance of the BOMM estimator $\hat{\mathbf{x}}_n^* = (\hat{x}_{n,1}^*, \cdots, \hat{x}_{n,d}^*)$ in \eqref{eq:bomm}, where $\hat{m}_l$ follows from \eqref{eq:margfunc} with $\phi$ and $f$ set as $\hat{\phi}$ and $\hat{f}(\mathbf{x}) = \hat{\phi} \{\hat{h}_1(x_1) + \cdots + \hat{h}_d(x_d)\}$, respectively. As $f$ is presumed to be a GAM, this reduces to $\hat{x}_{n,l}^* = \underset{x_l}{\text{argmin}} \; \hat{h}_l(x_l)$. The following theorem establishes its optimization rate:

\begin{theorem}\label{THM:consistency_NP_BOMM}
Under Assumptions \ref{assump:derivs_phi_h} -- \ref{assump:reg_parameter} above, the BOMM estimator $\hat{\mathbf{x}}_n^*$ in \eqref{eq:bomm} using the inference approach in \eqref{NONPARAM_EST_HOROWITZ} and \eqref{NONPARAM_EST_HOROWITZ2} yields the following optimization rate:
\begin{equation}
|f(\hat{\mathbf{x}}_n^*) - f(\mathbf{x}^*)| = \mathcal{O}_{P}\left(n^{-k/(4k+2)}\right),
\label{eq:bommrate}
\end{equation}
where constants in $\mathcal{O}_{P}$ may depend on $f$ and dimension $d$.
\end{theorem}
\noindent The proof of this theorem is provided in Appendix A.

Several useful insights can be gleaned from this theorem. First, as sample size $n \rightarrow \infty$, the optimality gap between the BOMM estimator $\hat{\mathbf{x}}_n^*$ and a global minimum $\mathbf{x}^*$ approaches zero, which proves the consistency of BOMM for global optimization. Second, as the degree of smoothness $k$ increases for the link and additive functions, the optimization rate in \eqref{eq:bommrate} also improves, which is not surprising. Finally and most importantly, the term in this rate relating to sample size $n$, namely $n^{-k/(4k+2)}$, does not depend on dimension $d$. This is in contrast to the $\mathcal{O}(n^{-\nu/d})$ optimization rate (discussed earlier in Section \ref{sec:exist}) for surrogate-based approaches using the Mat\'ern-$\nu$ GP, which deteriorates considerably as dimension $d$ increases. In this sense, BOMM can temper such a curse-of-dimensionality for existing black-box optimization methods. We show later that this translates to improved practical optimization performance over existing methods, for our target setting with limited runs in moderate-to-high dimensions.

\section{Practical Implementation}\label{sec:bommimpl}

With this theoretical foundation, we now present a practical framework for robust implementation of BOMM. We first leverage the transformed approximate additive GP in \cite{lin2020transformation} for probabilistic inference on the desired marginal mean functions to perform BOMM. We then propose a modification of BOMM, called BOMM+, for the setting where $h(\mathbf{x})$ may deviate from additivity. Finally, we provide convergence analysis for this GP-based implementation of BOMM and BOMM+.

\subsection{GP-based BOMM} \label{sec:mminf}
In what follows, we employ a (i) GP-based framework for inferring the model components in \eqref{eq:gam}. There are three reasons why this may be preferable to the (ii) optimization-based approach in \eqref{NONPARAM_EST_HOROWITZ}-\eqref{NONPARAM_EST_HOROWITZ2}. First, (ii) is largely used for theoretical analysis, and can be tricky to implement well as many hyperparameters need to be tuned. Second, (i) permits the \textit{probabilistic} inference of marginal mean functions, which we will leverage for a robust implementation of BOMM. Finally, the required smoothness conditions in \eqref{NONPARAM_EST_HOROWITZ}-\eqref{NONPARAM_EST_HOROWITZ2} can be imposed within (i) via a careful selection of GP kernels, as discussed next. We thus expect (i) to have a comparable optimization rate as shown for (ii) in Theorem \ref{THM:consistency_NP_BOMM}, although we prove just its consistency in Section \ref{sec:gpconv} for reasons discussed later.

To infer the model components in \eqref{eq:gam}, we adopt the transformed approximate additive GP (TAAG) in \cite{lin2020transformation}, which models $f$ as:
\begin{align}\label{eq:taag}
\begin{split}
f(\mathbf{x}) &= \phi_\lambda \left\{ A(\mathbf{x}) + Z(\mathbf{x}) \right\}, \quad A(\mathbf{x}) \sim \text{GP}\{\mu,k_{A}(\cdot,\cdot)\}, \quad Z(\mathbf{x}) \sim \text{GP}\{0,k_{Z}(\cdot,\cdot)\},\\
k_{A}(\mathbf{x},\mathbf{y}) & = \sigma^2(1-\eta)r_A(\mathbf{x}-\mathbf{y}), \quad r_A(\boldsymbol{\omega}) = \sum_{l=1}^d w_l r_{A,l}(\omega_l), \quad \sum_{l=1}^d w_l = 1,\\
k_Z(\mathbf{x},\mathbf{y}) &= \sigma^2 \eta r_Z(\mathbf{x}-\mathbf{y}),
\end{split}
\end{align}
where $\phi_\lambda$ is a link function parametrized by $\lambda$, and $A(\mathbf{x})$ and $Z(
\mathbf{x})$ are independent GPs. Here, $A(\mathbf{x})$ models the additive part of $h(\mathbf{x})$ in \eqref{eq:gam} via the additive kernel $r_A$ in \eqref{eq:taag}, where each additive term $r_{A,l}$ can be specified as a squared-exponential kernel or a Mat\'ern kernel that controls smoothness of the additive function $h_l$ in \eqref{eq:gam}. Next, $Z(\mathbf{x})$ models the residual non-additive part of $h(\mathbf{x})$ in \eqref{eq:gam}, namely $\zeta(\mathbf{x})$, via a zero-mean GP, where $r_Z$ is a non-additive kernel of choice. The parameter $\eta \in [0,1]$ controls the degree of non-additivity in $h(\mathbf{x})$: a near-zero value suggests that this function is near-additive, whereas a large value indicates considerable non-additivity. Finally, the parameter $\sigma^2 > 0$ serves as a global variance parameter on both $A(\mathbf{x})$ and $Z(\mathbf{x})$. 

For the link function $\phi_\lambda$, one choice (as adopted in \citealp{lin2020transformation}) is the well-known one-parameter Box-Cox transformation \citep{box1964analysis}. This can be defined as $\phi_\lambda^{-1} (z) = (1 - z^\lambda)/{\lambda}$ for $\lambda < 0$, $\phi_\lambda^{-1} (z) = \log z$ for $\lambda = 0$, and $(z^\lambda -1)/{\lambda}$ for $\lambda > 0$, where the parameter $\lambda \in \mathbb{R}$ is fit from data. Compared to its standard definition, the sign is flipped for the case of $\lambda < 0$ to ensure $\phi_{\lambda}$ is monotonically increasing; this does not affect its modeling capabilities. To use this transform, the black-box function $f$ needs to be strictly positive. This can be achieved in practice by adding an appropriately large constant on $f$, which does not affect its optimization. While one can employ a more flexible transformation choice (e.g., the two-parameter transform in \citealp{yeo2000new}), we find that the above Box-Cox transformation works quite well in later experiments.

With this, the marginal mean functions $\{m_l(x_l)\}_{l=1}^d$ can then be inferred as follows. Suppose we know the model parameters $\lambda$, $\mu$, $\sigma^2$, $\eta$ and $\mathbf{w} = (w_1, \cdots, w_d)$, along with the kernel length-scale parameters for $r_A$ and $r_Z$ (denoted as $\boldsymbol{\theta}_{A}$ and $\boldsymbol{\theta}_{Z}$, respectively); these will be estimated from data later. Denote the above parameter set by $\boldsymbol{\Theta}$. Recall from \eqref{eq:margfunc} that $m_l(x_l) = \int_{\mathcal{X}_{-l}} \phi_{\lambda}^{-1} \circ f(\mathbf{x}) \; d\mathbf{x}_{-l}$. Conditional on observed data $\mathbf{f}_n$, the following proposition shows that the posterior distribution of $m_l(\cdot)$ follows a Gaussian process:
\begin{proposition}\label{prop:marg_mean_dist}
Adopt the modeling framework in \eqref{eq:taag}, and suppose model parameters $\boldsymbol{\Theta}$ are known. Conditional on data $\mathbf{f}_n = [f(\mathbf{x}_1), \cdots, f(\mathbf{x}_n)]$, the marginal mean function $m_l(\cdot)$ has the posterior distribution $m_l(\cdot)|\mathbf{f}_n \sim \textup{GP}\{\mu_{n,l}(\cdot),k_{n,l}(\cdot,\cdot)\}$, where:
   \begin{align}
   \label{eq:postmarg1}
    \begin{split}
\mu_{n,l}(x_l) = \int_{\mathcal{X}_{-l}} \mu_{n, \phi_\lambda^{-1} \circ f}(\mathbf{x})d\mathbf{x}_{-l}, \quad k_{n,l}(x_l,x_l') = \int_{\mathcal{X}_{-l}}\int_{\mathcal{X}_{-l}} k_{n,\phi_\lambda^{-1} \circ f}(\mathbf{x}, \mathbf{x}') d\mathbf{x}_{-l}d\mathbf{x}'_{-l}.
    \end{split}
    \end{align}
\noindent Here, $\mu_{n, \phi_\lambda^{-1} \circ f}(\cdot)$ and $k_{n, \phi_\lambda^{-1} \circ f}(\cdot,\cdot)$ are the posterior mean and covariance functions of $\phi_\lambda^{-1} \circ f$ conditional on $\mathbf{f}_n$, given by:
\small
\begin{align}\label{eq:postmarg2}
\begin{split}
\mu_{n,\phi_\lambda^{-1} \circ f}(\mathbf{x}) &= \mu + \left( (1-\eta){\mathbf{r}_{n,A}}(\mathbf{x}) + \eta \mathbf{r}_{n,Z}(\mathbf{x}) \right)^\top \left((1-\eta)\mathbf{R}_{n,A} + \eta \mathbf{R}_{n,Z} \right)^{-1} \left(\phi_{\lambda}^{-1}(\mathbf{f}_n)-\mu \mathbf{1}\right),\\
k_{n,\phi_\lambda^{-1} \circ f}(\mathbf{x},\mathbf{x}') &= \sigma^2 \left(1-\tilde{\mathbf{r}}_n(\mathbf{x})^{\top}\left((1-\eta) \mathbf{R}_{n,A}+\eta \mathbf{R}_{n,Z}\right)^{-1} \tilde{\mathbf{r}}_n(\mathbf{x}') 
\right),
\end{split}
\end{align}
\normalsize
where $\mathbf{r}_{n,A}(\mathbf{x}) = [r_A(\mathbf{x}_i-\mathbf{x})]_{i=1}^n$, $\mathbf{r}_{n,Z}(\mathbf{x}) = [r_Z(\mathbf{x}_i-\mathbf{x})]_{i=1}^n$, $\tilde{\mathbf{r}}_n(\mathbf{x}) = (1-\eta)\mathbf{r}_{n,A}(\mathbf{x}) + \eta \mathbf{r}_{n,Z}(\mathbf{x})$, $\mathbf{R}_{n,A} = [r_A(\mathbf{x}_i-\mathbf{x}_j)]_{i,j=1}^n$ and $\mathbf{R}_{n,Z} = [r_Z(\mathbf{x}_i-\mathbf{x}_j)]_{i,j=1}^n$.
\end{proposition}
\noindent The proof of this proposition is provided in Appendix C. 

With this, the GP-based BOMM estimator then takes the form:
\begin{equation} \label{eq:bommgp}
\hat{\mathbf{x}}_n^* \coloneq (\hat{x}^*_{n,1}, \cdots, \hat{x}^*_{n,d}), \quad \quad \hat{x}^*_{n,l} = \underset{x_l}{\arg\min} \;  {\mu}_{n,l}(x_l), \quad \quad l = 1, \cdots, d.
\end{equation}
This can be further simplified when $\{r_{A,l}\}_{l=1}^d$ and $r_Z$ follow the squared-exponential form:
\begin{equation}\label{eq:kern_sqexp}
r_{A,l}(x_l,x_l') = \exp\left\{-\left(\frac{x_l-x_l'}{\theta_{A,l}} \right)^2 \right\}, \quad \quad r_{Z}(\mathbf{x},\mathbf{x}') = \exp\left\{-\sum_{l=1}^d \left(\frac{x_l-x_l}{\theta_{Z,l}} \right)^2 \right\}
\end{equation}
where $\boldsymbol{\theta}_A = (\theta_{A,1}, \cdots, \theta_{A,d})$ and $\boldsymbol{\theta}_Z = (\theta_{Z,1}, \cdots, \theta_{Z,d})$ are their length-scale parameters. With such kernels, the following proposition gives a closed-form objective for \eqref{eq:bommgp}:
\begin{proposition}\label{prop:marg_mean_sqexp}
Adopt the same conditions as Proposition \ref{prop:marg_mean_dist}. Under the squared-exponential kernels in \eqref{eq:kern_sqexp}, the BOMM estimator in \eqref{eq:bommgp} reduces to:
\small
\begin{equation}\label{eq:postmarg3}
\hat{x}^*_{n,l} = \underset{x_l}{\arg\min} \; \left[ (1- \eta) w_l \textup{Vol}(\mathcal{X}_{-l})\sum_{i=1}^n q_i\exp\left\{- \left( \frac{x_l-x_{i,l}}{\theta_{A,l}}\right)^2\right\} +  \pi^{\frac{d-1}{2}}\eta \sum_{i = 1}^n p_{i,l} q_i \left\{- \left( \frac{x_l-x_{i,l}}{\theta_{Z,l}}\right)^2\right\} \right],
\end{equation}
\normalsize
where $\mathbf{x}_i = (x_{i,1}, \cdots, x_{i,d})$ is the $i$-th design point. Here, $p_{i,l}$ and $\mathbf{q} = [q_1, \cdots, q_n]$ follow:
\begin{align*} \label{eq:piqi}
p_{i,l} = \prod_{j \neq l} \theta_{Z,j} \left(\tilde\Phi_{i,j}(U_j) - \tilde\Phi_{i,j}(L_j) \right), \quad \mathbf{q} = \left((1-\eta) \mathbf{R}_{n,A} +  \eta \mathbf{R}_{n,Z} \right)^{-1} \left(\phi_{\lambda}^{-1}(\mathbf{f}_n)-\mu \mathbf{1}\right),
\end{align*}
where $\tilde\Phi_{i,j}$ is the c.d.f. of $\mathcal{N}(x_{i,j},\theta_{Z,j}^2/2)$ and $\textup{Vol}(\mathcal{X}_{-l}) = \prod_{j \neq l} (U_j-L_j)$.
\end{proposition}
\noindent Similar expressions can be derived for other kernel choices, e.g., the Mat\'ern kernel, but may be more involved. With this closed-form objective, one can easily optimize the one-dimensional problem in \eqref{eq:postmarg3} (e.g., via grid search) to obtain the BOMM estimator $\hat{x}_{n,l}^*$. 

The above procedure, however, requires the estimation of parameters $\boldsymbol{\Theta}$ from data. To do this, we employ the same empirical Bayes approach as \cite{lin2020transformation}. This approach first assigns the following non-informative priors on model parameters $[\lambda,\mu,\tau^2,\delta,\mathbf{w},\boldsymbol{\theta}_{A},\boldsymbol{\theta}_{Z}] \propto 1$, where $\tau^2 = \sigma^2(1-\eta)$ and $\delta = \eta/(1-\eta)$ reparametrize $(\sigma,\eta)$. It then finds the fitted parameters $\hat{\boldsymbol{\Theta}}$ that maximize the corresponding marginal likelihood given observed data $\mathbf{f}_n$. Details on this procedure can be found in Section 3 of \cite{lin2020transformation}. With this in hand, the GP-based BOMM estimator \eqref{eq:bommgp} can then be computed using the plug-in estimate $\boldsymbol{\Theta} = \hat{\boldsymbol{\Theta}}$.

Algorithm \ref{ALG2:TAG_AM} summarizes each step of the GP-based BOMM optimization procedure, with a diagnostic procedure described later. First, the black-box simulator $f$ is evaluated at designed input points $\mathbf{x}_1, \cdots, \mathbf{x}_n$. In later experiments, we find that maximin Latin hypercube designs \citep{morris1995exploratory} work quite well: they not only provide desirable space-filling performance, but also offer good projective properties onto each input, which is important for accurate estimation of the additive structure in \eqref{eq:gam}. Next, one fits the transformed approximate additive GP in \cite{lin2020transformation}; our implementation makes use of the authors' \textsc{R} package \texttt{TAG} \citep{R_TAG}. With the fitted model, one then constructs the BOMM estimator via the optimization formulation \eqref{eq:bommgp}. In our implementation, this optimization is performed via one-dimensional grid searches. 

\begin{algorithm}[!t]
    \caption{GP-based BOMM+}
     \textbf{Input}: Sample size $n$ (from run budget), threshold $T$, significance level $\rho$
    \begin{algorithmic}[1]

        \State Construct a maximin Latin hypercube design $\{\mathbf{x}_i\}_{i=1}^n$, and evaluate $f$ on such points.
        
        \State Fit the transformed approximate additive GP in \cite{lin2020transformation} and obtain parameter estimates $\hat{\boldsymbol{\Theta}}$.

        \State Using $\hat{\boldsymbol{\Theta}}_{-\eta}$, compute the plug-in estimate of the posterior probability $\xi = \mathbb{P}(\eta > T|\hat{\boldsymbol{\Theta}}_{-\eta},\text{data})$ via \eqref{eq:eta}.

        \If{$\xi \leq 1-\rho$}

        \For{$l = 1,\cdots,d$}
        \State $\bullet$ Optimize the BOMM estimator $\hat{x}^*_{n,l}$ via \eqref{eq:bommgp}.
        \EndFor
        
        \Else{
        
        \For{$l = 1,\cdots,d$}
        \State $\bullet$ Specify the tail probability $\alpha^*$ following Appendix F.
        \State $\bullet$ Optimize the tail BOMM estimator $\hat{x}^*_{n,l} = \hat{x}^*_{n,\alpha^*,l}$ via \eqref{eq:bomtmgp}.
        \EndFor
        }
        \EndIf

   \noindent \hspace{-0.8cm} \textbf{Output}: $\hat{\textbf{x}}_n^* = (\hat{x}^*_{n,1}, \cdots, \hat{x}^*_{n,d})$
    \end{algorithmic}
     \label{ALG2:TAG_AM}
\end{algorithm}

\subsection{GP-based BOMM+}\label{subsec:atm}
Recall that a key motivation for BOMM is its use of marginal mean functions that can be estimated efficiently when $h = \phi^{-1} \circ f$ is near-additive, i.e., it has few interaction effects. When considerable interactions are present, a modification of BOMM using marginal \textit{tail} means can be used for robust performance. We present next a diagnostic approach for detecting such non-additivity, followed by a marginal tail means modification for estimating $\mathbf{x}^*$. We refer to this approach with diagnostic modification as BOMM+ hereafter.

Recall that the parameter $\eta$ dictates the level of non-additivity in the model \eqref{eq:taag}: the larger $\eta$ is, the greater its non-additivity. Thus, a reasonable diagnostic for non-additivity might be the posterior probability that $\eta$ is large, i.e., $\mathbb{P}(\eta > T|\text{data})$ for a desired threshold $T > 0$. Suppose for now that all model parameters in $\boldsymbol{\Theta}$ except $\eta$ (denoted $\boldsymbol{\Theta}_{-\eta}$) are known. Then the posterior distribution of $\eta$ takes the form:

\begin{equation}\label{eq:eta}
[\eta|\boldsymbol{\Theta}_{-\eta},\text{data}] \propto s^{-n} \eta^{ \delta}(1-\eta) \text{det}\{(1-\eta) \mathbf{R}_{n,A} + \eta \mathbf{R}_{n,Z} \}^{-\frac{1}{2}},
\end{equation}
where $s^2 = n^{-1}(\phi_{\lambda}^{-1}(\mathbf{f}_n) - {\mu}\mathbf{1})^\top \left\{(1-\eta) {\mathbf{R}}_{n,A} + \eta {\mathbf{R}}_{n,Z} \right\}^{-1} (\phi_{\lambda}^{-1}(\mathbf{f}_n) - {\mu}\mathbf{1})$. With this, the diagnostic probability $\mathbb{P}(\eta > T|\boldsymbol{\Theta}_{-\eta},\text{data})$ can be easily computed via Monte Carlo methods. In our later implementation, this is computed via the self-normalized importance sampling approach in Chapter 9 of \cite{owen2016monte}. One can then infer whether considerable non-additivity is present by seeing whether this probability is above a certain cut-off $1 - \rho$; in later experiments, we used a threshold of $T = 0.4$ and a significance level of $\rho=0.3$, which seemed to work well.

Of course, in practice the parameters $\boldsymbol{\Theta}_{-\eta}$ are unknown. From a Bayesian perspective, one would ideally sample from the full posterior distribution $[\boldsymbol{\Theta}|\text{data}]$, then marginalize over $\boldsymbol{\Theta}_{-\eta}$ to compute the diagnostic probability $\mathbb{P}(\eta > T|\text{data})$. Such a fully Bayesian approach, however, may be expensive given the many parameters in $\boldsymbol{\Theta}$. We adopt an alternate strategy using the empirical Bayes estimates $\hat{\boldsymbol{\Theta}}_{-\eta}$ from the previous subsection, which can be efficiently optimized via the \textsc{R} package \texttt{TAG} \citep{R_TAG}. In particular, we employ the plug-in estimator $\mathbb{P}(\eta > T|\hat{\boldsymbol{\Theta}}_{-\eta},\text{data})$, where given $\hat{\boldsymbol{\Theta}}_{-\eta}$, one can compute this probability from \eqref{eq:eta} via, e.g., self-normalized importance sampling \citep{owen2016monte}.

With this non-additivity diagnostic in hand, the BOMM estimator from the previous subsection should be used when $\mathbb{P}(\eta > T|\hat{\boldsymbol{\Theta}}_{-\eta},\text{data}) < 1 - \rho$, as this suggests there is near-additivity in $h$ that can be exploited. When this is not the case, there is some evidence for considerable non-additivity in $h$, in which case we adopt the following \textit{tail} marginal mean estimator. The intuition is as follows. Even when $h$ is not globally near-additive, its degree of additivity should increase as one hones in locally around its minimizer $\mathbf{x}^*$. Following \cite{mak2019analysis}, we employ a marginal tail means approach to exploit such \textit{local} additivity for optimization. In place of the posterior marginal mean ${\mu}_{n,l}(x) = \mathbb{E}[m_l(x)|\text{data}]$, we instead employ the posterior marginal tail mean ${\mu}_{n,l}^{[\alpha]}(x) = \mathbb{E}[m_l(x)|\text{data},m_l(x) \leq Q_{l}^{[\alpha]}(x)]$, where $Q_{l}^{[\alpha]}(x)$ is the $100\alpha\%$-percentile of the posterior distribution $[m_l(x)|\text{data}]$. Such a tail mean discards the top $100(1-\alpha)$\% of this posterior distribution before evaluating its expectation; this removes the part of the posterior that is more sensitive to large objective values in the data $\mathbf{f}_n$, allowing it to better exploit local additivity of $h$ near $\mathbf{x}^*$. Note that, with $\alpha = 1$, this reduces to the original posterior marginal mean $\mu_{n,l}(x)$. A similar tail means approach was employed in \cite{mak2019analysis} for discrete black-box optimization.

Since the posterior distribution of $m_l(x)$ is Gaussian (Proposition \ref{prop:marg_mean_dist}), its posterior tail mean function further admits the closed form $\mu_{n,l}^{[\alpha]}(x_l) = \mu_{n,l}(x_l) - \sqrt{k_{n,l}(x_l,x_l)}{\varphi\left(z_\alpha \right)}/{\alpha}$, where $z_\alpha$ is the $100\alpha\%$-percentile of the standard Gaussian and $\varphi$ is the standard Gaussian density. The resulting tail BOMM estimator is then given as:
\begin{equation} \label{eq:bomtmgp}
\hat{\mathbf{x}}_{n,\alpha}^* \coloneq (\hat{x}^*_{n,\alpha,1}, \cdots, \hat{x}^*_{n,\alpha,d}), \quad \quad \hat{x}^*_{n,\alpha,l} = \underset{x_l}{\arg\min} \;  {\mu}_{n,l}^{[\alpha]}(x_l), \quad \quad l = 1, \cdots, d.
\end{equation}
As before, one can employ the plug-in estimate $\boldsymbol{\Theta} = \hat{\boldsymbol{\Theta}}$ for evaluating \eqref{eq:bomtmgp}. We show later in experiments that such a tail estimator can provide robust optimization under non-additivity in $h$. Algorithm \ref{ALG2:TAG_AM} summarizes the full BOMM+ procedure with this non-additivity diagnostic. Appendix F provides guidance on how $\alpha$ can be selected in implementation.

\subsection{Optimization consistency}
\label{sec:gpconv}

We now establish the optimization consistency of the GP-based BOMM and BOMM+. The key difference between this analysis and that in Section \ref{sec:gapanalysis} lies in the considered function space for $f$. The earlier analysis from Section \ref{sec:gapanalysis} establishes an optimization rate for BOMM when $f$ follows a GAM, i.e., the model \eqref{eq:gam} with $\zeta(\mathbf{x}) = 0$. The following analysis shows that the GP-based BOMM and BOMM+ are consistent, i.e., its optimality gap $f(\hat{\mathbf{x}}_n^*) - f(\mathbf{x}^*)$ goes to zero, under mild deviations of $\zeta(\mathbf{x})$ from zero, i.e., under mild deviations from additivity for $h$ in \eqref{eq:gam}. Optimization rates for the GP-based BOMM and BOMM+ are more difficult to establish since there is little work on their corresponding function space; we will explore this as future work.

As before, suppose $\mathcal{X} = \prod_{l=1}^d [L_l,U_l]$. We presume the following form for $f$:
\begin{assumption}
\label{assump:rkhs}
The objective $f$ lies on the function space $\mathcal{F}_{\lambda}$, defined as:
\vspace{-0.5cm}
\begin{equation}\label{eq:flambda}
\mathcal{F}_{\lambda} = \{f: f = \phi_{\lambda} \circ h, \; h \in \mathcal{H}_{\rm TAAG}\}, \quad \lambda \in \mathbb{R}.
\vspace{-0.5cm}
\end{equation}

\end{assumption}
\noindent Here, $\mathcal{H}_{\rm TAAG}$ is the reproducing kernel Hilbert space (RKHS; \citealp{aronszajn1950theory}) of the kernel $k_A + k_Z$ on domain $\mathcal{X}$, corresponding to the GP for $A(\mathbf{x}) + Z(\mathbf{x})$ in \eqref{eq:taag}. This RKHS takes the form $\mathcal{H}_{\text{TAAG}}= \left\{h: h = h_A + h_Z,  h_A \in \mathcal{H}_{k_A}, h_Z \in \mathcal{H}_{k_Z} \right\}$ equipped with the norm $\|h\|_{\mathcal{H}_{\text{TAAG}}} = \underset{h = h_A + h_Z, h_A \in \mathcal{H}_{k_A}, h_Z \in \mathcal{H}_{k_Z}}{\min}\left(\|h_A\|_{\mathcal{H}_{k_A}} + \|h_Z\|_{\mathcal{H}_{k_Z}}  \right)$, where $\mathcal{H}_{k_A}$ and $\mathcal{H}_{k_Z}$ correspond to the RKHS for kernels $k_A$ and $k_Z$, respectively. Note that $\mathcal{H}_{\text{TAAG}}$ consists of functions that are \textit{non-additive} for $h$ due to the presence of $h_Z$.

We further make the following set of assumptions for theoretical analysis:

\begin{assumption}\label{assump:kernel}
The kernels $k_A$ and $k_Z$ in the RKHS $\mathcal{H}_{\textup{TAAG}}$ take the squared-exponential form \eqref{eq:taag} and \eqref{eq:kern_sqexp}. The GP modeling framework \eqref{eq:taag} for BOMM adopts the same kernels, with no misspecification of kernel hyperparameters or $\lambda$.
\end{assumption}

\begin{assumption}\label{assump:positive}
     The objective $f$ satisfies $f(\mathbf{x}) \in  [f^*,f^+]$ for $\mathbf{x} \in \mathcal{X}$, where $f^* = f(\mathbf{x}^*) > 0$ is the global minimum and $f^+ < \infty$ is an upper bound.
 \end{assumption}
\begin{assumption}\label{assump:unique}
    The objective $f$ admits a unique minimizer $\mathbf{x}^* \in \mathcal{X}$.
\end{assumption}

\begin{assumption}\label{assump:dom}
The objective $f$ satisfies the so-called ``first-order dominating'' condition: 
\begin{equation}\label{eq:fodom}
\arg\min_{\mathbf{x} \in \mathcal{X}} h(\mathbf{x}) = \arg\min_{\mathbf{x} \in \mathcal{X}} \sum_{l=1}^d \int h(\mathbf{x}) d\mathbf{x}_{-l}, \quad h = \phi_{\lambda}^{-1} \circ f.
\end{equation}  
\end{assumption}

\noindent Assumption \ref{assump:kernel} on kernel specification is typical for GP analysis (see, e.g., \citealp{ritter2000average}), although recent work \citep{wang2020prediction,wynne2021convergence} has explored the case of potential kernel misspecification. The consistency results later also hold for Mat\'ern kernels. Assumption \ref{assump:positive} is needed to ensure the Box-Cox transformation is valid; this is always possible by adding an appropriately large constant on $f$, which does not affect its optimization. Assumption \ref{assump:unique} is a mild condition on the uniqueness of $\mathbf{x}^*$. Assumption \ref{assump:dom} on the ``first-order dominating'' condition (a term we coined) permits mild interactions in $h(\mathbf{x})$, as long as its minimizer corresponds to that of its marginal mean functions; note that this holds naturally when $h(\mathbf{x})$ is additive. 

With this in hand, we now state the desired consistency result for the GP-based BOMM:
\begin{theorem}\label{thm:gpbommconv}
Under Assumptions \ref{assump:design_points} and \ref{assump:rkhs} -- \ref{assump:dom}, the BOMM estimator $\hat{\mathbf{x}}^*_n$ \eqref{eq:bommgp} using the GP modeling framework \eqref{eq:taag} satisfies $f(\hat{\mathbf{x}}^*_n) \overset{P}{\rightarrow} f(\mathbf{x}^*)$.
\end{theorem}
\noindent Its proof is provided in Appendix D. This theorem shows that, even when $f$ deviates mildly from generalized additivity (in that the first-order dominating condition \eqref{eq:fodom} still holds), the optimality gap for the GP-based BOMM converges to zero in probability as sample size $n$ increases, as desired. Here, the function space $\mathcal{F}_{\lambda}$ provides generalization on the GAM space considered earlier in Theorem \ref{THM:consistency_NP_BOMM}, which do not permit interaction effects in $h$.

We can further prove a similar consistency result for the GP-based BOMM+:
\begin{corollary}\label{corr:gpbomtmconv}
Under Assumptions \ref{assump:design_points} and \ref{assump:rkhs} -- \ref{assump:dom}, the BOMM+ estimator $\hat{\mathbf{x}}^*_{n}$ from Algorithm \ref{ALG2:TAG_AM} satisfies $f(\hat{\mathbf{x}}^*_{n}) \overset{P}{\rightarrow} f(\mathbf{x}^*)$.
\end{corollary}
\noindent Its proof is provided in Appendix E. In practice, as seen later in experiments, BOMM+ can have considerable improvements over existing methods even when $h(\mathbf{x})$ has moderate interactions. However, showing this via an optimization rate (as in Theorem \ref{THM:consistency_NP_BOMM}) is difficult for the broader function space in Theorem \ref{thm:gpbommconv}, as we have found little work on such a space.

\section{Numerical Experiments}\label{sec:numerics}

We now inspect the performance of the proposed BOMM+ approach compared to existing one-shot black-box optimization methods. We first outline the experiment set-up, then investigate the compared methods for a suite of test functions and a custom function where the degree of interactions can be controlled. Finally, we investigate a batch-sequential implementation of BOMM+ and compare it with an existing batch-sequential black-box approach. Such a batch-sequential setting is not the primary focus of this work, but we include this to demonstrate the potential of BOMM+ in broader settings.

\subsection{Experiment set-up}\label{subsec:numeric_setting}

We first give an overview of the compared methods in the following experiments:
\begin{itemize}
\item \textit{Pick-the-Winner \textup{(PW)}}: This is the simple benchmark of selecting the evaluated design point that yields the lowest observed objective, i.e., $\hat{\mathbf{x}}^*_n = \underset{\mathbf{x} \in \{\mathbf{x}_1, \cdots, \mathbf{x}_n\}}{\text{argmin}} \; f(\mathbf{x})$.
\item \textit{Surrogate-based-optimization via the squared-exponential GP \textup{(SBO-SqExp)}}: This is a standard SBO benchmark, using a GP surrogate with an anisotropic squared-exponential kernel (commonly used in computer experiments; see \citealp{gramacy2020surrogates}). All model parameters are estimated via maximum likelihood using the \textsc{R} package \texttt{DiceKriging} \citep{R_DiceKriging}. Its estimator for $\mathbf{x}^*$ takes the form $\hat{\mathbf{x}}^*_n = \underset{\mathbf{x} \in \mathcal{X}}{\text{argmin}} \; \hat{f}_n(\mathbf{x})$, where $\hat{f}_n(\cdot)$ is the posterior mean of the GP given data $\mathbf{f}_n$.
\item \textit{Surrogate-based-optimization via the deep GP \textup{(SBO-DGP)}}: SBO-DGP uses the above SBO approach, except the surrogate $\hat{f}_n(\cdot)$ uses the deep GP from \cite{sauer2023active}, fitted with the \textsc{R} package \texttt{deepgp} \citep{R_deepgp} and its default settings.
\item \textit{Surrogate-based-optimization via TAAG \textup{(SBO-TAAG)}}: SBO-TAAG uses the above SBO approach, except the surrogate $\hat{f}_n(\cdot)$ uses the TAAG model \eqref{eq:taag}, fitted with the \textsc{R} package \texttt{TAG} \citep{R_TAG} and its default settings. Another SBO benchmark, SBO-TAG, uses the transformed additive GP surrogate (model \eqref{eq:taag} with $\eta = 0$) fitted with the same package. While these are not common benchmarks, we include them to contrast our approach, which uses the marginal means estimator from the TAAG surrogate, with the direct optimization of such a surrogate.
\item \textit{BOMM+}: This is the proposed approach in Algorithm \ref{ALG2:TAG_AM} with threshold $T = 0.4$ and significance level $\rho = 0.3$. 
\end{itemize}

\noindent All methods use the same points, sampled from a maximin Latin hypercube design \citep{morris1995exploratory} from the \textsc{R} package \texttt{lhs} \citep{R_lhs}. The sample size $n$ is set as $10d$, following \cite{loeppky2009choosing}. To quantify simulation variability, each experiment is replicated 20 times. The considered methods are compared on their optimality gap $f(\hat{\mathbf{x}}_n^*) - f(\mathbf{x}^*)$, i.e., the objective gap between its predicted minimizer and the true minimizer.

\subsection{A simulation bake-off}

\begin{figure}[!t]
    \centering
    \includegraphics[width=0.75\textwidth]{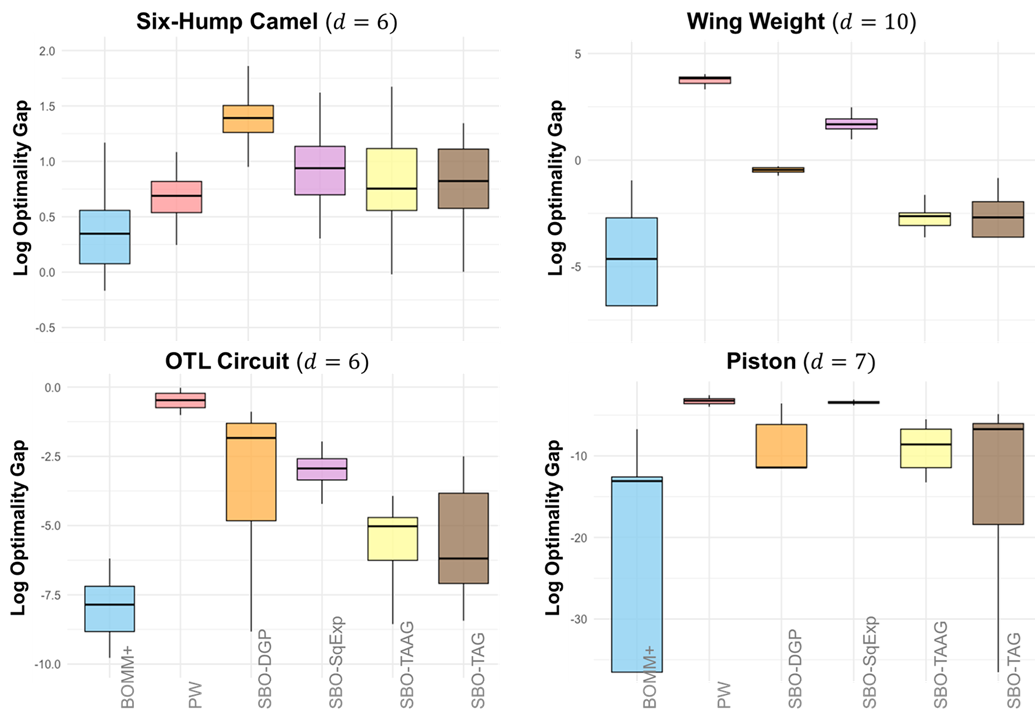}

    \caption{Log-optimality gaps of the compared methods for the six-hump camel, wing weight, OTL circuit and piston functions. Boxplots show experiment variability over 20 replications.}
    \label{Fig:SIMUL_VIRTUAL_LIB_EXP}
\end{figure}

With this set-up, we investigate a simulation ``bake-off'' of the compared methods in a suite of test functions in the computer experiments literature. In addition to the six-hump and wing weight functions from Section \ref{subsec:motiv_prob}, we consider two more test functions from \cite{surjanovic2013virtual}: the OTL circuit function in $d=6$ dimensions, and the piston function in $d=7$ dimensions; their specific forms are provided in Appendix G.

\begin{wrapfigure}{r}{0.35\textwidth}
\vspace{-0.6cm}
    \centering
    \includegraphics[width=\linewidth]{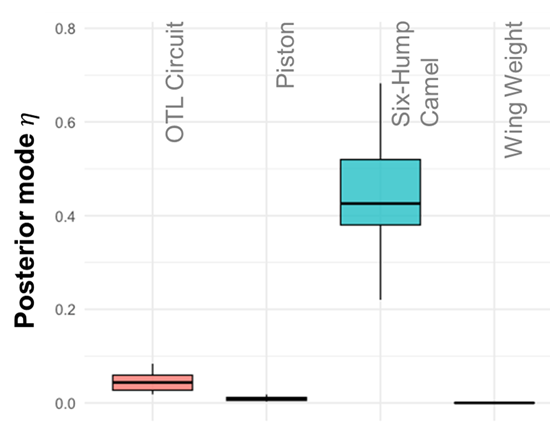}
    \caption{For BOMM+, the posterior mode of $[\eta|\hat{\boldsymbol{\Theta}}_{-\eta},\textup{data}]$ over 20 replications for the compared functions.}
    \label{Fig:diag}
    \vspace{-0.6cm}
\end{wrapfigure}

Figure \ref{Fig:SIMUL_VIRTUAL_LIB_EXP} shows the boxplots of the log-optimality gaps for each of the four test functions. There are several useful observations. First, the same limitations of existing methods noted in Section \ref{subsec:motiv_prob} arise here. In selecting $\hat{\mathbf{x}}_n^*$ amongst evaluated points, PW yields mediocre performance particularly as dimension $d$ increases. SBO approaches with the standard squared-exponential GP (SBO-SqExp) and deep GP (SBO-DGP) yield improvements over PW in some cases; in other cases, they may perform considerably worse. One reason is again its reliance on a good global surrogate fit on $\mathcal{X}$; when this is poor, such methods may perform worse than PW. Our BOMM+ approach performs quite well; it yields considerably smaller optimality gaps compared to other methods for all functions. Figure \ref{Fig:diag} shows boxplots of the estimated $\hat{\eta}$ (taken as the posterior mode of $[\eta|\hat{\boldsymbol{\Theta}}_{-\eta},\text{data}]$) for its underlying TAAG model. We see that the OTL, piston and wing weight functions have near-zero $\hat{\eta}$, suggesting (i) the presence of latent near-additive structure after transformation; the six-hump camel has considerably larger $\hat{\eta}$, suggesting (ii) the presence of latent interaction effects in $h$. For (i), BOMM+ employs the BOMM estimator \eqref{eq:bommgp} to exploit such near-additive structure via marginal mean functions. For (ii), BOMM+ employs the tail BOMM estimator \eqref{eq:bomtmgp}, which exploits local near-additivity via marginal tail means. In doing so, BOMM+ enjoys improved optimization performance over existing methods given limited runs in moderate-to-high dimensional domains.

\begin{figure}[!t]
  \centering
  \includegraphics[width=\textwidth]{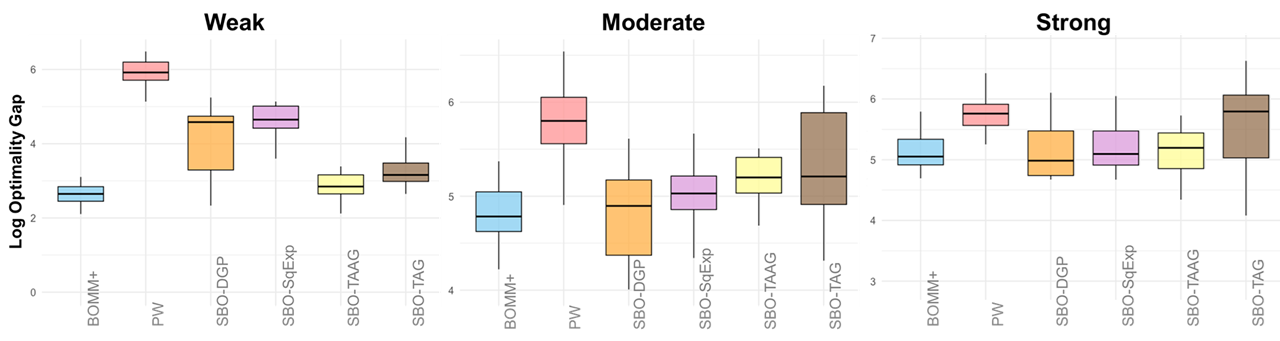}
  \caption{Log-optimality gaps for the weak ($\lambda=0.05$), moderate ($\lambda=0.3$) and strong ($\lambda=0.5$) interaction cases of the test function \eqref{eq:expfn}. Boxplots show experiment variability over 20 replications.}
  \label{Fig:SIMUL_NONADDITIVITY}
\end{figure}

The contrast between BOMM+ and the SBO approaches SBO-TAAG and SBO-TAG deserves further discussion. The latter approaches directly optimize various forms of the fitted TAAG model \eqref{eq:taag}, whereas BOMM+ makes use of the marginal mean functions from this fitted model. We see that, by \textit{modeling} for latent near-additive structure, SBO-TAAG and SBO-TAG offer some improvements over existing benchmarks. However, by further leveraging such latent near-additivity via a marginal means \textit{estimator} of $\mathbf{x}^*$, BOMM+ can further exploit this structure to yield considerably reduced optimization gaps. Given the challenges of limited samples in high-dimensional domains, this highlights the importance of fully exploiting marginal structure via BOMM+ for effective black-box optimization.

Next, we investigate the effectiveness of the diagnostic in Section \ref{subsec:atm} via the following $d=9$-dimensional custom test function, which is based on the exponential test function in \cite{dette2010generalized}. For brevity, $\epsilon$ and $\{m_l\}_{l=1}^9$ are specified in Appendix G.
\small
\begin{align}\label{eq:expfn}
&f(\textbf{x}) = 10\sum_{l=0}^{d/3-1} \sum_{m=1}^3 e^{-2/x_{3l+m}^{(m+1)/2}+\epsilon}
 + \lambda\sum_{l=1}^{d/3}\left\{(x_{3l-2}-m_{3l-2})-(x_{3l-1}-m_{3l-1})-(x_{3l}-m_{3l})\right\}^2,\nonumber\\
&(x_1, x_3, x_5) \in [0,5]^3, ~(x_2, x_8) \in [1,6]^2,~x_4 \in [1.5, 6.5],~ (x_6, x_7, x_9) \in [2,7]^3.
\end{align}
\normalsize
Here, the first term in \eqref{eq:expfn} is additive, and its second term controls the magnitude of interaction effects; the larger $\lambda > 0$ is, the greater such interactions. We inspect three functions with different interaction levels: $\lambda = 0.05$ (weak), $\lambda = 0.3$ (moderate) and $\lambda = 0.5$ (strong). The same methods are compared under the same settings, with 20 replications.

\begin{wrapfigure}{r}{0.4\textwidth}
\vspace{-0.5cm}
\small
    \begin{sc}
    \begin{tabular}{lc}
    \toprule
     & Percentage \\
    \midrule
    Weak ($\lambda=0.05$) & $0\%$ \\
    Moderate ($\lambda=0.3$) & $95\%$ \\
    Strong ($\lambda=0.5$) & $100\%$ \\
    \bottomrule
    \end{tabular}    
    \end{sc}
    \normalsize
    \captionof{table}{\label{table:diag2} Percentage of replications for which the BOMM+ diagnostic detects considerable non-additivity on $h$ for the test function \eqref{eq:expfn}. }
    \vspace{-0.5cm}
\end{wrapfigure}

Figure \ref{Fig:SIMUL_NONADDITIVITY} shows the boxplots of the log-optimality gaps for each function, and Table \ref{table:diag2} shows the percentage of replications for which considerable non-additivity is detected on $h$ via the diagnostic in Section \ref{subsec:atm}. For the weak interaction case, the diagnostic correctly identifies the lack of considerable non-additivity in all replications; BOMM+ then leverages the marginal means estimator \eqref{eq:bommgp} to exploit such structure, yielding improved optimization over benchmarks. For the moderate and strong cases, the diagnostic identifies considerable non-additivity in nearly all replications; BOMM+ then uses the marginal tail means estimator to exploit local additivity, yielding comparable or better performance to the best benchmarks. Here, SBO-DGP also performs well in the moderate and strong cases, with comparable optimality gaps to BOMM+. However, as noted before, such a method may suffer from a lack of robustness: when its surrogate fits poorly over $\mathcal{X}$, its optimization can be worse than PW (see Figure \ref{Fig:SIMUL_VIRTUAL_LIB_EXP}).

\subsection{Batch-sequential BOMM+} \label{subsec:seq}

\begin{wrapfigure}{r}{0.5\textwidth}
\vspace{-0.5cm}
\centering
\includegraphics[width=\linewidth]{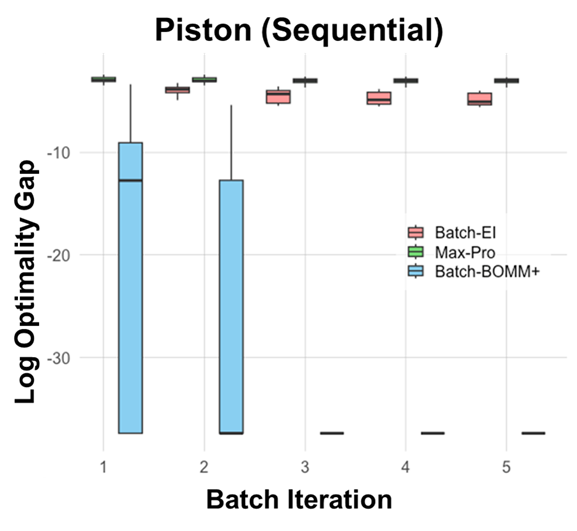}
\caption{Log-optimality gaps of the compared batch-sequential methods for the piston function as a function of batch iteration. Boxplots show experimental variability over 20 replications.}
\label{Fig:SIMUL_SEQ}
\vspace{-0.5cm}
\end{wrapfigure}

Suppose $f$ is evaluated at a set of initial design points $\mathbf{x}_1, \cdots, \mathbf{x}_n$. We wish to use this to adaptively select the next batch of points $\mathbf{x}_{n+1}, \cdots, \mathbf{x}_{n+b}$ for minimizing $f$, where $b > 1$ is the batch size. Consider the following simple approach. First, select one of the $b$ points as the inferred solution $\hat{\mathbf{x}}_n^*$ from BOMM+ using current evaluations of $f$ as data. Next, select the remaining $b-1$ points from a random Latin hypercube design (LHD; \citealp{mckay2000comparison}). The objective $f$ is then evaluated on this batch of design points, the TAAG model is re-fit, and the above batch procedure is repeated for $m \geq 1$ iterations (or until the run budget is exhausted). This can be intuited by the well-known exploration-exploitation trade-off \citep{kearns2002near}: the $b-1$ LHD points target the \textit{exploration} of $f$ to identify latent near-additive structure, and the evaluation at the BOMM+ estimate $\hat{\mathbf{x}}^*_n$ targets the \textit{exploitation} of this learned structure for optimization via marginal means.

As a proof-of-concept, we test this batch-sequential approach (which we call Batch-BOMM+) on the earlier $d=7$ piston function. Here, $n_{\rm ini}=35$ maximin LHD points are used initially, then batches of $b=5$ runs are taken until a total budget of $n=70$ evaluations is exhausted. We compare with two standard benchmarks. The first is a simple batch-sequential space-filling design approach using the maximum projection design in \cite{joseph2015maximum}, as implemented in the \textsc{R} package \texttt{MaxPro} \citep{R_MaxPro}. This can be viewed as a ``pure exploration'' strategy. The second is the batch expected improvement approach (Batch-EI; \citealp{chevalier2013fast}), which is widely used for batch-sequential Bayesian optimization. Here, Batch-EI uses a GP surrogate with anisotropic squared-exponential kernel, which is re-fit at each batch iteration. This simulation is replicated 20 times. Figure \ref{Fig:SIMUL_SEQ} shows the log-optimality gaps for the compared methods at each batch iteration. We see that Batch-BOMM+ yields consistent improvements over the two benchmarks as batch iteration increases. This shows that a simple adaptive optimization approach that leverages learned latent near-additive structure can be promising in a batch-sequential setting; we will investigate this as future work.

\section{Application: Neutrino detector optimization}\label{sec:application}

\begin{figure}[!t]
\centering
\includegraphics[width=\textwidth]{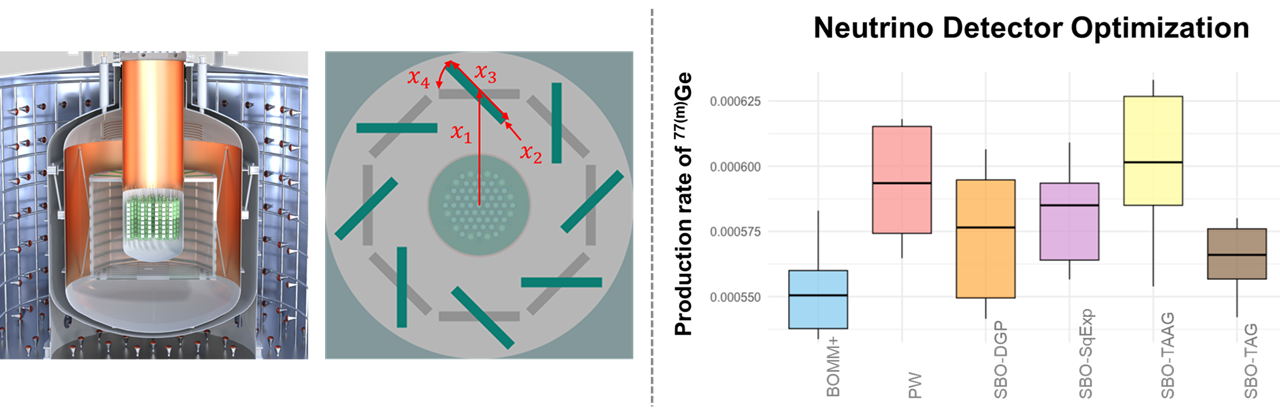}
\caption{[Left] The neutrino detector schematic for the LEGEND project, along with the considered neutron moderator geometry with $d=4$ inputs. [Right] Comparison of ${}^{77(m)}$\textup{Ge} production rates (smaller-the-better) for the selected moderator designs from each method. Boxplots show experiment variability over 10 replications.}
\label{Fig:Neutrino}
\end{figure}

The search for neutrinoless double-beta decay ($0\nu\beta\beta$) is a frontier in modern physics \citep{lrp}; if detected, this decay could provide an explanation for the matter-antimatter asymmetry \citep{canetti2012matter}, where there is a greater abundance of matter over antimatter in the Universe. The LEGEND (Large Enriched Germanium Experiment for Neutrinoless Double-Beta Decay) project \citep{legend1000pCDR} searches for $0\nu\beta\beta$ decay in a massive liquid argon cryostat in which 1000~kg of ${}^{76}$Ge-enriched germanium detectors are immersed (Figure \ref{Fig:Neutrino} left). A key experimental challenge is to minimize the cosmogenic neutron background generated by high-energy cosmic muons \citep{Pandola2007}. Such muons can enter the experiment and generate secondary neutrons, which interact with ${}^{76}$Ge to produce unwanted isotopes (e.g., ${}^{77(m)}$Ge) \citep{meierhofer2010}. The decays of such isotopes could mimic $0\nu\beta\beta$ events and thus obscure the desired physics signals \citep{wiesinger2018}.

To mitigate this background, one strategy is to employ a neutron moderator that slows down or absorbs the undesirable neutrons before they reach the inner, sensitive germanium detectors \citep{neuberger2021, schuetz2025resum}. Designing an effective moderator is challenging: it must suppress the flux of neutrons while remaining compatible with demanding engineering and material constraints. We investigate here a turbine-like moderator geometry (Figure \ref{Fig:Neutrino} middle), in which eight polyethylene panels are arranged radially around the germanium detector array to enhance the panels' directional shielding performance. This geometry is parametrized by $d=4$ inputs with corresponding ranges: the turbine radius $x_1$ (180-230 cm), the panel thickness $x_2$ (10-15 cm), the panel length $x_3$ (100-150 cm), and the panel tilt angle $x_4$ (0-20 degrees). The goal is to optimize the moderator design $\mathbf{x}$ within this geometry range for effective neutron shielding by minimizing $f(\mathbf{x})$, the production rate of the unwanted isotope ${}^{77(m)}$Ge.

A key challenge for this optimization is the simulation cost of a single moderator design. A high-fidelity simulation of this shielding process requires modeling individual primary muons and their interactions in the rock overburden and the shielding, which can require hundreds of CPU hours and is thus too expensive for method comparison. Instead, as a proof-of-concept, we use a lower-fidelity simulator \citep{warwick-ramachers,warwick-moritz} that injects secondary neutrons directly as primaries within the liquid argon cryostat, which focuses computational resources on the critical neutron transport within the active detector region. Each run of this lower-fidelity simulator requires 1 CPU hour, which facilitates method comparison. Here, the same methods as Section \ref{sec:numerics} are compared, with all methods using the same $n=50$ design points from a maximin Latin hypercube design. This experiment is replicated 10 times, and performance is gauged on the production rate of ${}^{77(m)}$Ge (smaller-the-better) for the selected moderator designs. 

Figure \ref{Fig:Neutrino} (right) shows the boxplots of ${}^{77(m)}$Ge production rates for the selected moderator designs from each method. As before, we see that PW yields mediocre performance, which is expected since the evaluated points are likely far from optimal. The SBO benchmarks give mixed results: some offer slightly lower ${}^{77(m)}$Ge rates to PW, whereas others yield slightly higher rates. BOMM+ again improves upon existing benchmarks, which highlights the importance of exploiting latent near-additive structure via marginal means for enhancing black-box optimization given limited experimental runs. It should be noted that, for neutrino detector design, optimization metrics at the upper tail percentiles are also of interest, as one wants to ensure good shielding performance with high confidence. From Figure \ref{Fig:Neutrino} (right), BOMM+ and SBO-TAG provide the best performance at the 90\% percentile (top whisker of boxplot). Our approach, however, has greater potential for identifying detector designs with improved shielding over SBO-TAG, as indicated by other percentiles in the boxplots.

\section{Conclusion}\label{sec:conclusion}

This paper introduces a new Black-box Optimization via Marginal Means (BOMM) method for effective one-shot optimization of an expensive black-box function $f$. Existing methods, e.g., pick-the-winner and surrogate-based optimization approaches, may yield mediocre performance with poor robustness, particularly as input dimensionality increases. To address this, BOMM leverages a new estimator of a global optimizer using marginal mean functions, which can be effectively estimated in high dimensions with limited runs. We prove that, when $f$ follows a generalized additive model and under mild conditions, the optimality gap from BOMM converges to zero and at a rate with considerably less dependence on dimensionality than existing methods. We then present a practical framework for implementing BOMM using the transformed approximate additive GP in \cite{lin2020transformation}, and prove its consistency for black-box optimization. Numerical experiments and an application to neutrino detector design demonstrate the improved black-box optimization performance of BOMM over existing methods with limited runs in moderate-to-high dimensions.

Given these promising results, there are several directions for further investigation. First, we will explore broader function spaces (e.g., extensions of the additive multi-index GP in \cite{li2023additive}) on which marginal structure can similarly be exploited for optimization. Next, given the promising results in Section \ref{subsec:seq}, we will develop an adaptive implementation of BOMM that sequentially exploits marginal structure, and investigate its theoretical properties. Finally, we will investigate a multi-fidelity extension of BOMM to fully tackle the neutrino detector design application using high-fidelity simulators.

\if0\blind{
}
\fi
\newpage
\appendix
\section{Proof of Theorem 1}
\label{pf:consistency_NP_BOMM}

As before, we suppose $\mathcal{X} = \prod_{l=1}^d [L_l,U_l]$. To prove Theorem 1 of the main paper, we first require the following lemmas. 

\begin{lemma}[Theorem 2.2 of \citep{horowitz2007rate}]\label{thm:L2_rate}
    Under Assumptions 1 -- 4 of the main paper, we have:
    \begin{align*}
        &\left\|\widehat \phi \circ \left(\widehat h_1 + \cdots + \widehat h_d\right) - \phi\circ \left(h_1 + \cdots +  h_d \right) \right\|_{L^2}^2 \\ &= \int_{\mathcal{X}} \left[\widehat \phi \circ \left(\widehat h_1(x_1) + \cdots + \widehat h_d(x_d)\right) - \phi\circ \left(h_1(x_1) + \cdots +  h_d(x_d) \right)\right]^2 d\mathbf{x} =  \mathcal{O}_p\left(n^{-2k/(2k+1)}\right).
    \end{align*}
\label{lem:horowitz}
\end{lemma}
\noindent We can generalize this lemma to establish the following uniform convergence result:
\begin{lemma}\label{thm:L_inf_rate}
    Under Assumptions 1 -- 4 of the main paper, we have:
    \begin{align*}
     &\left\|\widehat \phi \circ \left(\widehat h_1 + \cdots + \widehat h_d\right) - \phi\circ \left(h_1 + \cdots +  h_d \right) \right\|_{L^\infty} =  \mathcal{O}_p\left(n^{-k/(4k+2)} \right).
    \end{align*}
\end{lemma}

\begin{proof}[Proof (Lemma 2)]
Let us define:
\begin{align*}
    \psi &:= \phi\circ \left(h_1 + \cdots +  h_d \right) \\
    \widehat\psi &: = \widehat \phi \circ \left(\widehat h_1 + \cdots + \widehat h_d\right).
\end{align*}
From Assumption 1, it follows that $\widehat \psi - \psi \in W^{1,2}(\mathcal{X})$. From the Sobolev Embedding Theorem (Theorem 12.71 in \citep{hunter2001applied}), we get that:
    \begin{align*}
        \left\|\widehat \psi - \psi \right\|^2_{L^\infty} &\le C \left \|\widehat \psi - \psi \right \|^2_{W^{1,2}(\mathcal{X})} \\ &= C\left(\left \|\widehat \psi - \psi \right \|_{L^2}^2 + \sum_{|\alpha| = 1}\left \|D^\alpha\widehat \psi - D^\alpha \psi \right \|_{L^2}^2 \right)\\
        &= \mathcal{O}\left(\left\|\widehat \psi - \psi\right\|^{2}_{L^2} \right) + \mathcal{O}\left(\sum_{|\alpha| = 1} \left\|D^\alpha\widehat \psi - D^\alpha \psi \right \|_{L^2}^2\right),
    \end{align*}
    for some constant $C > 0$. With $|\alpha|$ = 1 and $|\tilde\alpha| = 2$, from the Gagliardo-Nirenberg inequality (Theorem 1 in \citep{nirenberg1966extended}), we have:
    \begin{align*}
           \left\|D^\alpha\widehat \psi - D^\alpha \psi \right\|_{L^2} &\leq C_1 \left\|D^{\tilde\alpha} \widehat \psi - D^{\tilde\alpha} \psi \right\|_{L^2}^{\frac{1}{2}} \left\|\widehat \psi - \psi \right\|_{L^2}^{\frac{1}{2}} + C_2 \left\|\widehat \psi - \psi\right\|_{L^2} \\
           &= \mathcal{O}\left( \left\|\widehat \psi - \psi \right\|_{L^2}^{\frac{1}{2}} + \left\|\widehat \psi - \psi \right\|_{L^2}\right) \\ 
           &= \mathcal{O}_p\left( n^{-k/(4k+2)}\right),
    \end{align*}
    for some constants $C_1 > 0$ and $C_2 > 0$. The first equality follows from the fact that $D^{\tilde\alpha} \widehat \psi - D^{\tilde\alpha} \psi$ is continuous on a bounded domain $\mathcal{X}$, and the second equality is a consequence of Lemma \ref{lem:horowitz}. Therefore, we have:
    \begin{align*}
        \left\|\widehat \psi - \psi \right\|^2_{L^\infty}  &=  \mathcal{O}\left(\left\|\widehat \psi - \psi\right\|^{2}_{L^2} \right) +  \mathcal{O}\left(\sum_{|\alpha| = 1} \left\|D^\alpha\widehat \psi - D^\alpha \psi \right \|_{L^2}^2\right) \\
        &=  \mathcal{O}_p\left(n^{-2k/(2k+1)} \right) +  \mathcal{O}_p\left(n^{-k/(2k+1)}\right) \\
        &= \mathcal{O}_p\left(n^{-k/(2k+1)}\right),
    \end{align*}
    which yields the statement.
\end{proof}

With this, we can now prove Theorem 1 of the main paper.
\begin{proof}[Proof (Theorem 1)]
Recall that
$f(\mathbf{x}) = \phi(h_1(x_1) + \cdots + h_d(x_d))$. Define:
\begin{equation*}
\widehat f(\mathbf{x}) := \widehat\phi\left(\widehat h_1(x_1) + \cdots + \widehat h_d(x_d)\right),
\end{equation*}
where $\widehat \phi$, $\widehat h_1, \cdots \widehat h_d$ are solutions to the constrained least squares problem in Equations (7)--(8) of the main paper. Note that:
\begin{align*}
0 &\le f\left(\widehat{\mathbf{x}}_n^{*}\right) - f(\mathbf{x}^*) \\
&= f\left(\widehat{\mathbf{x}}_n^{*}\right) - \widehat f\left(\widehat{\mathbf{x}}_n^{*}\right) + \widehat f\left(\widehat{\mathbf{x}}_n^{*}\right) - \widehat f\left(\mathbf{x}^*\right) + \widehat f\left(\mathbf{x}^*\right)  - f(\mathbf{x}^*) \\
&\le f\left(\widehat{\mathbf{x}}_n^{*}\right) - \widehat f\left(\widehat{\mathbf{x}}_n^{*}\right) + \widehat f\left(\mathbf{x}^*\right)  - f(\mathbf{x}^*),
\end{align*}    
where the last inequality follows from the fact $\widehat f\left(\widehat{\mathbf{x}}_n^{*}\right) - \widehat f\left(\mathbf{x}^*\right) \le 0$. From Lemma 2, we have:
\begin{align*}
    f\left(\widehat{\mathbf{x}}_n^{*}\right) - \widehat f\left(\widehat{\mathbf{x}}_n^{*}\right) &= \mathcal{O}_p\left(n^{-k/(4k+2)} \right),\\
    \widehat f\left(\mathbf{x}^*\right)  - f(\mathbf{x}^*) &=  \mathcal{O}_p\left(n^{-k/(4k+2)} \right),
\end{align*}
which proves the statement.

\end{proof}

\section{Proof of Proposition 2}
\label{pf:marg_mean_dist}
\begin{proof}
    Since the linear functional of the Gaussian process remains a Gaussian process \citep{bogachev1998gaussian}, it is enough to show the posterior mean and covariance functions for $m_l(x_l) = \int h(\mathbf{x})d\mathbf{x}_{-l}$. Note that the posterior mean and covariance functions for $h(\mathbf{x}) = \phi_{\lambda}^{-1}\circ f(\mathbf{x})$, given in Equation (12) of the main paper, follow directly from the GP predictive equations (Equation (2) of main paper); we denote these as $\mu_{n,h}(\mathbf{x}) = \int_{\Omega} h(\mathbf{x};\omega) \mathbb{P}(d\omega)$ and $k_{n,h}(\mathbf{x},\mathbf{x}')$, where $\mathbb{P}$ denotes the posterior measure on $h$ given data. For the posterior mean function of $m_l$, using Fubini's theorem, we have:
    \begin{align*}
         \mathbb{E}_{\mathbb{P}}\left[\int_{\prod_{j \neq l} [L_j,U_j]} h(\mathbf{x})d\mathbf{x}_{-l}\right ]& = \int_\Omega \int_{\prod_{j \neq l} [L_j,U_j]} h(\mathbf{x}; \omega)d\mathbf{x}_{-l} \mathbb{P}(d\omega) \\
         &= \int_{\prod_{j \neq l} [L_j,U_j]}  \int_\Omega h(\mathbf{x}; \omega) \mathbb{P}(d\omega)d\mathbf{x}_{-l}\\
         &= \int_{\prod_{j \neq l} [L_j,U_j]}\mu_{n,h}(\mathbf{x}) d\mathbf{x}_{-l}.
    \end{align*}
    For its posterior covariance function, note that:
    \begin{align*}
         &\text{Cov}_{\mathbb{P}}\left[\int_{\prod_{j \neq l} [L_j,U_j]} h(\mathbf{x})d\mathbf{x}_{-l}, ~\int_{\prod_{j \neq l} [L_j,U_j]} h(\mathbf{x}')d\mathbf{x}'_{-l}\right ] \\
         &= \int_\Omega \left(\int_{\prod_{j \neq l} [L_j,U_j]} h(\mathbf{x}; \omega)-\mu_{n,h}(\mathbf{x})d\mathbf{x}_{-l} \right)\left(\int_{\prod_{j \neq l} [L_j,U_j]} h(\mathbf{x}'; \omega)-\mu_{n,h}(\mathbf{x}')d\mathbf{x}'_{-l} \right)\mathbb{P}(d\omega) \\
         &= \int_\Omega \int_{\prod_{j \neq l} [L_j,U_j]}\int_{\prod_{j \neq l} [L_j,U_j]} (h(\mathbf{x}'; \omega)-\mu_{n,h}(\mathbf{x}'))(h(\mathbf{x}; \omega)-\mu_{n,h}(\mathbf{x}))d\mathbf{x}_{-l} d\mathbf{x}'_{-l} \mathbb{P}(d\omega) \\
         &= \int_{\prod_{j \neq l} [L_j,U_j]}\int_{\prod_{j \neq l} [L_j,U_j]}\int_\Omega  (h(\mathbf{x}'; \omega)-\mu_{n,h}(\mathbf{x}'))(h(\mathbf{x}; \omega)-\mu_{n,h}(\mathbf{x}))\mathbb{P}(d\omega)d\mathbf{x}_{-l} d\mathbf{x}'_{-l} \\
         &= \int_{\prod_{j \neq l} [L_j,U_j]}\int_{\prod_{j \neq l} [L_j,U_j]}k_{n,h}(\mathbf{x}, \mathbf{x}')d\mathbf{x}_{-l} d\mathbf{x}'_{-l},
    \end{align*}
    which proves the statement.
\end{proof}

\section{Proof of Proposition 3}
\label{pf:marg_mean_sqexp}
\noindent Notice that:
\begin{align*}
&\int_{\prod_{j \neq l} [L_j,U_j]} \mu_{n,\phi_\lambda^{-1} \circ f}(\mathbf{x}) d\mathbf{x}_{-l} \\&= \int _{\prod_{j \neq l} [L_j,U_j]}\mu + \left( (1-\eta){\mathbf{r}_{n,A}}(\mathbf{x}) + \eta \mathbf{r}_{n,Z}(\mathbf{x}) \right)^\top \left((1-\eta)\mathbf{R}_{n,A} + \eta \mathbf{R}_{n,Z} \right)^{-1} \left(\phi_{\lambda}^{-1}(\mathbf{f}_n)-\mu \mathbf{1}\right) d\mathbf{x}_{-l}\\
&= \mu \prod_{j \neq l} (U_j-L_j)+\int _{\prod_{j \neq l} [L_j,U_j]} \sum_{i = 1}^n q_i \left[ (1-\eta){\mathbf{r}_{n,A}}(\mathbf{x}) + \eta \mathbf{r}_{n,Z}(\mathbf{x})\right]_i d\mathbf{x}_{-l},
\end{align*}
where $q_i$ is the $i^{\text{th}}$ coordinate of 
$\left((1-\eta)\mathbf{R}_{n,A} + \eta \mathbf{R}_{n,Z} \right)^{-1} \left(\phi_{\lambda}^{-1}(\mathbf{f}_n)-\mu \mathbf{1}\right)$,
and:
\begin{align*}
\left[ (1-\eta){\mathbf{r}_{n,A}}(\mathbf{x}) + \eta \mathbf{r}_{n,Z}(\mathbf{x})\right]_i = (1-\eta)r_A(\mathbf{x}_i-\mathbf{x}) + \eta r_Z(\mathbf{x}_i-\mathbf{x}).
\end{align*}
\noindent Also note that:
\begin{align*}
    \int _{\prod_{j \neq l} [L_j,U_j]} \sum_{i = 1}^n q_i \left[ (1-\eta){\mathbf{r}_{n,A}}(\mathbf{x}) + \eta \mathbf{r}_{n,Z}(\mathbf{x})\right]_i d\mathbf{x}_{-l} 
    &= (1-\eta) \sum_{i = 1}^n q_i  \int _{\prod_{j \neq l} [L_j,U_j]}  r_A(\mathbf{x}_i-\mathbf{x})d\mathbf{x}_{-l} \\ &\quad + \eta \sum_{i = 1}^n q_i  \int _{\prod_{j \neq l} [L_j,U_j]} r_Z(\mathbf{x}_i-\mathbf{x}) d\mathbf{x}_{-l}.
\end{align*}
In fact, we can derive explicit formulae to compute these integrals by exploiting the structure of $r_A$ and $r_Z$. We first focus on the expression that involves $r_A$. Using the additive structure of $r_A$ with $\sum_{k=1}^n w_k = 1,  w_k \ge 0$ and $\mathbf{x}_i = (x_{i,1}, \cdots, x_{i,d})^\top$, we have:
\begin{align*}
    \int_{\prod_{j \neq l} [L_j,U_j]}  r_A(\mathbf{x}-\mathbf{x}_i)d\mathbf{x}_{-l} &= \sum_{k=1}^d   w_k \int _{\prod_{j \neq l} [L_j,U_j]}  \exp(-(x_k-x_{i,k})^2/\theta_{A,k}^2)d\mathbf{x}_{-l},
\end{align*}
and it can be shown that:
\begin{align*}
&\int _{\prod_{j \neq l} [L_j,U_j]}  \exp(-(x_k-x_{i,k})^2/\theta_{A,k}^2)d\mathbf{x}_{-l} \\
&= \begin{cases}
         \exp(-(x_j-x_{i,l})^2/\theta_{A,l}^2) \cdot \prod_{j \neq l} (U_j-L_j)   \quad \text{for}~~ k = l,\\
       \int_{[L_k,U_k]} \exp\left(-(x_k-x_{i,k})^2/\theta_{A,k}^2\right)dx_{k} \cdot \prod_{j \neq l, k} (U_j-L_j)  \quad \text{for}~~ k \neq l.
    \end{cases}
\end{align*}
In particular,  we have:
$$
\int_{[L_k,U_k]} \exp(-(x_k-x_{i,k})^2/\theta_{A,k}^2)dx_{k} 
 =  \sqrt{\pi}\theta_{A,k} \left(\Phi_{i,k}(U_k) - \Phi_{i,k}(L_k) \right),
$$
where $\Phi_{i,k}$ is the cumulative distribution corresponding to $N\left(x_{i,k}, \theta_{A,k}^2/2\right)$. Therefore, we have:
\begin{align*}
\int _{\prod_{j \neq l} [L_j,U_j]}  r_A(\mathbf{x}-\mathbf{x}_i)d\mathbf{x}_{-l} &= w_l \exp(-(x_l-x_{i,l})^2/\theta_{A,l}^2) \prod_{j \neq l} (U_j-L_j) \\ &+ \sqrt{\pi}\sum_{k\neq l} w_k  \theta_{A,k}  \left(\Phi_{i,k}(U_k) - \Phi_{i,k}(L_k) \right)  \prod_{ j \neq l, k} (U_j-L_j) ,
\end{align*}
which leads to:
$$
(1-\eta) \sum_{i = 1}^n q_i \int _{\prod_{j \neq l} [L_j,U_j]}  r_A(\mathbf{x}-\mathbf{x}_i)d\mathbf{x}_{-l} = (1- \eta)  w_l  \prod_{j \neq l} (U_j-L_j) \cdot \sum_{i=1}^n q_i\exp(-(x_l-x_{i,l})^2/\theta_{A,l}^2) + C_l,
$$
for some constant $C_l> 0$. For the second term, note that:
\begin{align*}
   \int _{\prod_{j \neq l} [L_j,U_j]} r_Z(\mathbf{x}-\mathbf{x}_i) d\mathbf{x}_{-l} &=  \int _{\prod_{j \neq l} [L_j,U_j]} \prod_{j=1}^d \exp\left(-(x_j-x_{i,j})^2/\theta_{Z,j}^2\right) d\mathbf{x}_{-l} \\
   &= \exp(-(x_l-x_{i,l})^2/\theta_{Z,l}^2) \cdot \prod_{j \neq l} \int_{[L_j,U_j]} \exp(-(x_j-x_{i,j})^2/\theta_{Z,j}^2)dx_j \\
   &= \exp(-(x_l-x_{i,l})^2/\theta_{Z,l}^2) \cdot \pi^{\frac{d-1}{2}}\prod_{j \neq l} \theta_{Z,j} \left(\tilde\Phi_{i,j}(U_j) - \tilde\Phi_{i,j}(L_j) \right),
\end{align*}
where $\tilde\Phi_{i,j}$ is the cumulative distribution corresponding to $N(x_{i,j}, \theta_{Z,j}^2/2)$.
Combining these together, it follows that $\arg\min_{x_l \in [L_l,U_l]} \int_{\prod_{j \neq l} [L_j,U_j]} \mu_{n,\phi_\lambda^{-1} \circ f}(\mathbf{x}) d\mathbf{x}_{-l}$ is equivalent to:
\small
\begin{align*} \arg\min_{x_l \in [L_l,U_l]} \left[ (1- \eta)  w_l \text{Vol}(\mathcal{X}_{-l}) \sum_{i=1}^n q_i\exp\left(-(x_l-x_{i,l}\right)^2/\theta_{A,l}^2) +  \pi^{\frac{d-1}{2}}\eta \sum_{i = 1}^n p_{i,l} q_i \exp(-\left(x_l-x_{i,l})^2/\theta_{Z,l}^2\right) \right],
\end{align*}
\normalsize
where $\text{Vol}(\mathcal{X}_{-l}) = \prod_{j \neq l} (U_j-L_j)$ and $p_{i,l} = \prod_{j \neq l} \theta_{Z,j} \left(\tilde\Phi_{i,j}(U_j) - \tilde\Phi_{i,j}(L_j) \right)$. This proves the statement.

\section{Proof of Theorem 4}
\label{pf:gpbommconv}

\begin{proof}
Denote $h(\mathbf{x}) = \phi_{\lambda}^{-1}\circ f(\mathbf{x)}$, and its posterior mean and variance as $\mu_{n,h}(\mathbf{x})$ and $k_{n,h}(\mathbf{x},\mathbf{x})$ (see Equation (12) of the main paper). To ease presentation, denote the minimizer as $\mathbf{x}^* = (x^*_1, \cdots, x^*_d)$ and the BOMM estimator to be $\widehat{\mathbf{x}}^{\text{BOMM}} \coloneq (\widehat x_1, \cdots, \widehat x_d)$. We prove the desired convergence result in the following two steps.\\

\noindent \textbf{Step 1:} 
    We first show that the marginal mean of $h$ can be well-approximated by that of $\mu_{n, h}(\mathbf{x})$. More precisely, we establish, for all $j = 1, \cdots, d$:
    \begin{equation}\label{eqn:marginal_effect_conv}
        \sup_{x_j \in [L_j, U_j]} \left | \int h(\mathbf{x}) d\mathbf{x}_{-j} - \int \mu_{n, h}(\mathbf{x}) d\mathbf{x}_{-j} \right| = \mathcal{O}\left(\exp\left(-C/d_n \right) \right),
    \end{equation}
    where $d_n := \sup_{\mathbf{x} \in \mathcal{X}}\inf_{\mathbf{x'} \in \{\mathbf{x}_1, \cdots, \mathbf{x}_n\}} \|\mathbf{x}-\mathbf{x}'\|$ is the so-called fill-distance in the kriging and kernel interpolation literature \citep{wendland2004scattered}.
    To show this, recall that $h \in \mathcal{H}_{\text{TAAG}}$. Then, by Corollary 3.11 in \citep{kanagawa2018gaussian}, we first have:
\begin{equation}\label{eqn:inter_pol_error}
       | h(\mathbf{x}) -  \mu_{n,h}(\mathbf{x}) | \le \|h\|_{\mathcal{H}_{\text{TAAG}}}\sqrt{V_{n, h}(\mathbf{x})}, ~~ \text{where} ~~ V_{n, h}(\mathbf{x}):= k_{n,h}(\mathbf{x},\mathbf{x})
    \end{equation}
    for all $\mathbf{x} \in \mathcal{X}$.
    Furthermore, since the kernel $k_{\text{TAAG}}$ is infinitely differentiable, by Theorem 11.22 in \citep{wendland2004scattered}, we know that, for large enough $n$:
    \begin{equation}\label{eqn:var_decay}
        \sup_{\mathbf{x} \in \mathcal{X}} \sqrt{V_{n,h}(\mathbf{x})} \lesssim \exp\left(-C/d_n \right),
    \end{equation}
    where the constant $C$ is independent of $n$. Combining \eqref{eqn:inter_pol_error} and \eqref{eqn:var_decay}, we have, for large enough $n$:
\begin{equation}\label{inter_error_bound}
    \sup_{\mathbf{x} \in \mathcal{X}} | h(\mathbf{x}) -  \mu_{n,h}(\mathbf{x}) |  \lesssim  \|h\|_{\mathcal{H}_{\text{TAAG}}}\exp\left(-C/d_n \right).
    \end{equation}
    Back to \eqref{eqn:marginal_effect_conv}, for any $j \in \{1, \cdots, d\}$, we have, for large enough $n$:
    \begin{align*}
         \sup_{x_j \in [L_j, U_j]} \left | \int h(\mathbf{x}) d\mathbf{x}_{-j} - \int \mu_{n,h}(\mathbf{x}) d\mathbf{x}_{-j} \right| \le \int 
     \sup_{\mathbf{x} \in \mathcal{X}}\left|h - \mu_{n,h}\right|d\mathbf{x}_{-j} \lesssim \|h\|_{\mathcal{H}_{\text{TAAG}}}\exp\left(-C/d_n \right).
    \end{align*} 
    \noindent \textbf{Step 2:} We next establish convergence of $\widehat{\mathbf{x}}^{\text{BOMM}} \coloneq (\widehat x_1, \cdots, \widehat x_d)$ to $\mathbf{x}^* = (x^*_1, \cdots, x^*_d)$ and $h(\widehat{\mathbf{x}}^{\text{BOMM}})$ to $h(\mathbf{x}^*)$. Let us define the following notation:
    $$
       m_j(x_j) \coloneqq \int h(\mathbf{x}) d\mathbf{x}_{-j}, \quad  \widehat m_j(x_j) \coloneqq \int \mu_{n,h}(\mathbf{x}) d\mathbf{x}_{-j}.
    $$
    Using the first-order dominating condition in Assumption 9, we know that $x^*_j$ minimizes $m_j(x_j)$. From Step 1 (see \eqref{eqn:marginal_effect_conv}), we have, for some constant $C > 0$ with large enough $n > 0$:
    \begin{equation}
    \left| m_j(\widehat x_j)-\widehat m_j(\widehat x_j) \right| \le  \sup_{x_j \in [L_j, U_j]} \left |\int h(\mathbf{x}) d\mathbf{x}_{-j}-\int \mu_{n,h}(\mathbf{x}) d\mathbf{x}_{-j} \right| = \mathcal{O}\left(\exp\left(-C/d_n\right) \right).\label{ineq_m_first}
    \end{equation}
    Similarly, by the same logic:
    \begin{equation}
    \left|\widehat m_j( x_j^*) - m_j( x_j^*) \right| \le  \sup_{x_j \in [L_j, U_j]} \left |\int \mu_{n,h}(\mathbf{x}) d\mathbf{x}_{-j}-\int h(\mathbf{x}) d\mathbf{x}_{-j} \right|  = \mathcal{O}\left(\exp\left(-C/d_n\right) \right).\label{ineq_m_second}
    \end{equation}
    Hence, for all $j \in \{1, \cdots, d\}$ with large enough $n>0$, we have:
    \begin{align*}
    0 &\le m_j(\widehat x_j) - m_j(x^*_j)\\ &= m_j(\widehat x_j) - \widehat m_j(\widehat x_j) + \widehat m_j(\widehat x_j) - \widehat m_j(x^*_j) + \widehat m_j(x^*_j) - m_j(x^*_j) \\
    &= \mathcal{O}\left(\exp\left(-C/d_n\right)\right) + \widehat m_j(\widehat x_j) - \widehat m_j(x^*_j) + \mathcal{O}\left(\exp\left(-C/d_n\right)\right),
\end{align*}
where the first inequality follows from the definition of $x^*_j$, and the last equality comes from \eqref{ineq_m_first} and \eqref{ineq_m_second}. Since $\widehat m_j(\widehat x_j) - \widehat m_j(x^*_j) \le 0$ by the definition of $\widehat x_j$, we deduce that:
$$
\left|m_j( \widehat x_j) - m_j(x^*_j)\right| = \mathcal{O}\left(\exp\left(-C/d_n\right)\right).
$$
Under Assumption 3, we know that the fill-distance $d_n$ converges to zero in probability as $n$ increases \citep{oates2019convergence,helin2023introduction}, yielding $m_j( \widehat x_j)\overset{P}{\rightarrow}m_j(x^*_j)$ for all $j \in \{1, \cdots, d\}$. Furthermore, as $h \in \mathcal{H}_{\text{TAAG}}$, $m_j$ is a continuous function on a closed interval $[L_j, U_j]$. Then, from the uniqueness of $\mathbf{x}^*$ (Assumption 8), $\widehat{\mathbf{x}}^{\text{BOMM}}\overset{P}{\rightarrow}\mathbf{x}^*$ follows. To see this, let $\epsilon > 0$ and consider a set $B_j := \{x_j: |x_j-x_j^*| \ge \epsilon \}$. Due to the uniqueness of $\mathbf{x}^*$, we know that $\inf_{x_j \in B_j} m_j(x_j) - m_j(x_j^*) \ge \eta$, for some $\eta > 0$. Therefore, for all $j \in \{1, \cdots, d\}$, we have:
$$
\mathbb{P}\left(|\hat x_j-x_j^*| \ge \epsilon\right) \le \mathbb{P}\left(m_j(\hat x_j) - m_j(x_j^*) \ge \eta/2\right) \overset{P}{\rightarrow} 0 ~~ \text{as} ~~ n\to\infty.
$$
Moreover, from the continuity of $h$ and $\phi_\lambda$, we obtain $h(\widehat{\mathbf{x}}^{\text{BOMM}})\overset{P}{\rightarrow}h(\mathbf{x}^*)$ as well as $f(\widehat{\mathbf{x}}^{\text{BOMM}})= \phi_\lambda\circ h(\widehat{\mathbf{x}}^{\text{BOMM}})\overset{P}{\rightarrow}f(\mathbf{x}^*) = \phi_\lambda\circ h(\mathbf{x}^*)$, which proves the claim.
\end{proof}

\section{Proof of Corollary 1}
\label{pf:gpbomtmconv}
\begin{proof}
To avoid confusion, let us denote the tail BOMM estimator (from Algorithm 1 of the main paper) as $\widehat{\mathbf{x}}^{\text{TBOMM}}$ and the BOMM estimator as $\widehat{\mathbf{x}}^{\text{BOMM}}$. Observe that:
\begin{align*}
    0 &\le h\left(\widehat{\mathbf{x}}^{\text{TBOMM}}\right) - h(\mathbf{x}^*)\\ &= h\left(\widehat{\mathbf{x}}^{\text{TBOMM}}\right) - \mu_{n,h}\left(\widehat{\mathbf{x}}^{\text{TBOMM}}\right) + \mu_{n,h}\left(\widehat{\mathbf{x}}^{\text{TBOMM}}\right) - \mu_{n,h}\left(\widehat{\mathbf{x}}^{\text{BOMM}}\right) \\
    &\quad + \mu_{n,h}\left(\widehat{\mathbf{x}}^{\text{BOMM}}\right) - h\left(\widehat{\mathbf{x}}^{\text{BOMM}}\right) + h\left(\widehat{\mathbf{x}}^{\text{BOMM}}\right) - 
    h(\mathbf{x}^*) \\
    &\le h\left(\widehat{\mathbf{x}}^{\text{TBOMM}}\right) - \mu_{n,h}\left(\widehat{\mathbf{x}}^{\text{TBOMM}}\right) + \mu_{n,h}\left(\widehat{\mathbf{x}}^{\text{BOMM}}\right) - h\left(\widehat{\mathbf{x}}^{\text{BOMM}}\right) + h\left(\widehat{\mathbf{x}}^{\text{BOMM}}\right) - 
    h(\mathbf{x}^*),
\end{align*}
where we used the identity 
$\mu_{n,h}\left(\widehat{\mathbf{x}}^{\text{TBOMM}}\right) \le \mu_{n,h}\left(\widehat{\mathbf{x}}^{\text{BOMM}}\right)$ in the last inequality, which holds using the specification rule for $\alpha$ in Appendix \ref{app:alpha}. From \eqref{inter_error_bound}, for large enough $n$, we know that:
\begin{align*}
\left|h\left(\widehat{\mathbf{x}}^{\text{TBOMM}}\right) - \mu_{n,h}\left(\widehat{\mathbf{x}}^{\text{TBOMM}}\right)\right| &= \mathcal{O}\left(\exp\left(-C/d_n\right) \right) \\
\left|\mu_{n,h}\left(\widehat{\mathbf{x}}^{\text{BOMM}}\right) - h\left(\widehat{\mathbf{x}}^{\text{BOMM}}\right)\right| &= \mathcal{O}\left(\exp\left(-C/d_n\right) \right).
\end{align*}
This gives us:
\begin{equation}\label{BOMTM_h_bound}
0 \le h\left(\widehat{\mathbf{x}}^{\text{TBOMM}}\right) - h(\mathbf{x}^*) = \mathcal{O}\left(\exp\left(-C/d_n\right) \right) + h\left(\widehat{\mathbf{x}}^{\text{BOMM}}\right) - h(\mathbf{x}^*).
\end{equation}
From Theorem 4, we observe that the right-most term of \eqref{BOMTM_h_bound} converges to zero in probability as $n\to \infty$. Under Assumption 3, we know the fill-distance $d_n\overset{P}{\rightarrow}0$ in probability as $n \to \infty$. And therefore, we have
$h\left(\widehat{\mathbf{x}}^{\text{TBOMM}}\right) \overset{P}{\rightarrow}h(\mathbf{x}^*)$. From the continuity of $\phi_\lambda$, we can further show the convergence of $f\left(\widehat{\mathbf{x}}^{\text{TBOMM}}\right) = \phi_\lambda \circ h(\widehat{\mathbf{x}}^{\text{TBOMM}}) \overset{P}{\rightarrow}f(\mathbf{x}^*) = \phi_\lambda \circ h(\mathbf{x}^*)$, as desired.
\end{proof}

\section{Selection of tail probability $\alpha$ in BOMM+}
\label{app:alpha}

In our experiments, we employ the following strategy for selecting the tail probability $\alpha$ in the tail BOMM estimator (see Equation (17) of the main paper). The idea is to choose $\alpha$ such that the predicted response of $h$ (and thus $f$) at the tail BOMM estimator $\hat{\mathbf{x}}_{n,\alpha}^*$ is minimized. In other words, this uses the fitted model to calibrate a good choice of $\alpha$ that effectively leverages local additivity for minimization; a similar idea was used in \citealp{mak2019analysis} for discrete optimization. Formally, $\alpha$ is selected as:
\begin{equation}
\alpha^* = \argmin_{\alpha \in (0,1]} \mu_{n,h}(\hat{\mathbf{x}}_{n,\alpha}^*),
\end{equation}
where $\mu_{n,h}$ is the posterior mean of $h$ (see Equation (12) of the main paper). This one-dimensional optimization is performed via grid search in our numerical experiments.

\section{Test function specification}
\label{sec:test}
We provide below the detailed specification of test functions in numerical experiments:
\begin{itemize}
\item The six-hump camel function \citep{molga2005test} in $d=6$ dimensions:
\begin{align*}
\begin{split}
f(\textbf{x}) &= \sum_{k=1}^{3} \left( \left( 4 - 2.1 x_{2k-1}^2 + \frac{x_{2k-1}^4}{3} \right) x_{2k-1}^2 + x_{2k-1} x_{2k} + (-4 + 4 x_{2k}^2) x_{2k}^2 \right) + 5,\\
&x_1, x_3, x_5 \in [-2,2], \quad x_2, x_4, x_6 \in [-1,1].
\vspace{-0.5cm}
\end{split}
\normalsize
\end{align*}
\item The wing weight function \citep{moon2010design} in $d=10$ dimensions:
\footnotesize
\begin{align*}
\begin{split}
    f(\textbf{x}) &= 0.036 S_w^{0.758} W_{fw}^{0.0035} \left( \frac{A}{\cos^2(\Lambda)} \right)^{0.6} q^{0.006} \lambda^{0.04} \left( \frac{100 t_c}{\cos(\Lambda)} \right)^{-0.3}(N_z W_{dg})^{0.49} + S_w  W_p, \\
        \textbf{x} &= (S_w, W_{fw}, A, \Lambda, q, \lambda, t_c, N_z, W_{dg}, W_p)\\
S_w &\in [150,200],\; W_{fw} \in [220,300],\; A \in [6,10],\; \Lambda \in [-10,10],\; q \in [16,45],\\
&\lambda \in [0.5,1],\; t_c \in [0.08,0.18],\; N_z \in [2.5,6],\; W_{dg} \in [1700,2500],\; W_p \in [0.025,0.08].
\end{split}
\end{align*}
\normalsize
    \item The OTL circuit function \citep{moon2012two} in $d=6$ dimensions:
    \begin{align*}\label{eq:otl}
\begin{split}
f(\textbf{x}) &= \frac{(V_{b1}+0.74)\,\beta\,(R_{c2}+9)}{\beta\,(R_{c2}+9)+R_f} + \frac{11.35\,R_f}{\beta\,(R_{c2}+9)+R_f} + \frac{0.74\,R_f\,\beta\,(R_{c2}+9)}{\bigl[\beta\,(R_{c2}+9)+R_f\bigr]\,R_{c1}},\\
        \mathbf{x} &= (R_{b1},\,R_{b2},\,R_f,\,R_{c1},\,R_{c2},\,\beta), \quad V_{b1} = \frac{12\,R_{b2}}{R_{b1}+R_{b2}}, \\
         &R_{b1}\in[50,150],\; R_{b2}\in[25,75],\; R_{f}\in[0.5,3],\\
&R_{c1}\in[1.2,2.5],\; R_{c2}\in[0.25,1.2],\; \beta\in[50,300].
\end{split}
\vspace{-0.5mm}
\end{align*}
\item The piston simulation function \citep{moon2010design} in $d=7$ dimensions:
\begin{align*}
\begin{split}
    f(\mathbf{x}) &= 2\pi \sqrt{\frac{M}{\,k + S^2\,\frac{P_0\,V_0}{T_0}\,\frac{T_a}{V^2}\,}}, \quad \mathbf{x} = (M,\,S,\,V_0,\,k,\,P_0,\,T_a,\,T_0),\\
        V &= \frac{S}{2k}\Bigl(\sqrt{\,A^2 + 4\,k\,\frac{P_0\,V_0}{T_0}\,T_a\,}
        \;-\;A\Bigr), \quad A = P_0 S + 19.62\,M \;-\; \frac{k\,V_0}{S}, \\
        &M\in[30,60],\; S\in[0.005,0.020],\; V_0\in[0.002,0.010],\; k\in[1000,5000],\\
&P_0\in[90\,000,110\,000],\; T_a\in[290,296],\; T_0\in[340,360].
\end{split}
\vspace{-0.5mm}
\end{align*}
\item Custom test function in $d=9$ dimensions (Equation (20) of the main paper): $\epsilon = 0.01$, $m_1 = m_3 = m_5 = 2.5$, $m_2 = m_8 = 3.5$, $m_4 = 4$, $m_6 = m_7 = m_9 = 4.5$.
\end{itemize}

\spacingset{0.94}
\small
\newpage
\bibliography{references}

\end{document}